\newif\ifsubmission
\submissionfalse          
\ifsubmission
  \documentclass[letterpaper]{article} 
  \usepackage[submission]{aaai2027}    
  \usepackage[hyphens]{url}            
  \usepackage{graphicx}                
  \usepackage{natbib}                  
  \usepackage{caption}                 
  \usepackage{booktabs}
\else
  \documentclass[11pt]{article}
  \usepackage[margin=1in]{geometry}
  \usepackage{graphicx}
  \usepackage{booktabs}
  \usepackage{microtype}
  \usepackage[colorlinks=true,linkcolor=blue,citecolor=blue,urlcolor=blue]{hyperref}
  \hypersetup{
    pdftitle={The Value of a Prompt: An LLM-Relative Kolmogorov-Complexity Approach},
    pdfauthor={Rafael Pass}
  }
\fi

\usepackage{amsmath,amssymb,amsthm}
\usepackage{comment}

\ifsubmission
  \excludecomment{proof}                 
  \newcommand{\fullonly}[1]{}
  \newcommand{\subonly}[1]{#1}
\else
  \newcommand{\fullonly}[1]{#1}
  \newcommand{\subonly}[1]{}
\fi

\newtheorem{theorem}{Theorem}[section]
\newtheorem{lemma}[theorem]{Lemma}
\newtheorem{proposition}[theorem]{Proposition}
\newtheorem{corollary}[theorem]{Corollary}
\theoremstyle{definition}
\newtheorem{definition}[theorem]{Definition}

\theoremstyle{remark}
\newtheorem{remark}[theorem]{Remark}

\newcommand{\EOT}{\mathrm{EOT}}
\newcommand{\E}{\mathbb{E}}
\newcommand{\Val}{\mathrm{Val}}
\newcommand{\eps}{\epsilon}
\newcommand{\EOS}{\mathrm{EOS}}

\newcommand{\ACC}{\texttt{ACC}}
\newcommand{\TC}{\mathrm{TokenCost}}
\DeclareMathOperator{\med}{med}
\newcommand{\REJ}{\texttt{REJ}}
\newcommand{\pKt}{\mathrm{pKt}}
\newcommand{\pKtt}{\widetilde{\pKt}}

\title{The Value of a Prompt:\\
An LLM-Relative Kolmogorov-Complexity Approach}
  \author{Rafael Pass\footnote{Rafael Pass is supported in part by AFOSR Award FA9550-24-1-0267, ISF Award 2338/23 and ERC Advanced Grant KolmoCrypt - 101142322. Any opinions, findings and conclusions or recommendations expressed in this material are those of the author(s) and do not necessarily reflect the views of the United States Government, the AFOSR, the European Union or the European Research Council Executive Agency.
}\\Cornell Tech, Technion, TAU}
  \date{\today}

\begin{document}
\maketitle

\begin{abstract}
In a world where valuable artifacts are increasingly created, completed, or
processed by LLMs, the central economic question is not only what the LLM can
produce, but what \emph{value} remains in the inputs (i.e., the prompts) we
provide to it. Given a prompt, hint, critique, problem statement, or partial
solution that helps an LLM produce an artifact $z$---a proof, program, design,
or scientific hypothesis---how should we measure the value of that input?

Intuitively, an input is valuable when it makes the target artifact easier for
the model to generate: either by increasing its sampling probability, or by
reducing the thinking time needed to find it. We propose a computational
Levin--Kolmogorov complexity approach to this problem, by appropriately
replacing the universal Turing machine in the classical definitions by the LLM
itself. Concretely, we introduce an LLM-relative notion of \emph{probabilistic
Levin--Kolmogorov complexity} $\pKt$---treating the model's thinking as the
random tape of the program, and charging logarithmically for it in Levin's
manner---and define prompt value as algorithmic mutual information with respect
to $\pKt$. This captures the intuition above: a prompt having $b$ bits of value
for an artifact $z$ makes $z$ $2^b$ times ``easier to obtain'', by
multiplying the success probability by $2^b$, by dividing the required
computation by $2^b$, or by any corresponding tradeoff between probability and
computation.

In contrast to the classical notion of algorithmic mutual information, ours is efficiently estimable.
We additionally show that, under a natural reproduction experiment, a prompt
value of \(b\) bits means that reproducing \(z\) without the prompt has median
token cost \(2^b\) times that of reproducing it with the prompt.
\end{abstract}

\newpage
\section{Introduction}
Suppose a large language model (LLM) produces a valuable \emph{artifact},
represented for our purposes by a string \(z\)---for example, a proof of a
mathematical theorem, a computer program, a design, a drug, or a scientific
hypothesis.  The LLM operates in some fixed deployment context and receives an
additional input, referred to as the \emph{prompt} \(p\).\footnote{We use
``prompt'' broadly to mean any type of input supplied externally, whether by a
human or another resource.  For instance, if an LLM were directly connected to
a human brain, the resulting brain signals could be viewed as the ``prompt.''}
The prompt might be a hint, critique, example, or partial solution.  A natural
question arises:
\begin{quote}
\emph{How much value did the prompt contribute to the final artifact, relative to
what the same LLM could have produced without it?}
\end{quote}

This is increasingly an economic question. As LLM capabilities become
abundant, the value of human contributions may lie in choosing the right input
to provide to the model. Simply measuring the length of that input---for
example, by its token count---does not capture its value: a short hint can be
decisive, whereas a long prompt can be irrelevant or even harmful. Rather, a
useful measure should (1) compare the difficulty of producing the artifact
\emph{with} and \emph{without} the input, and (2) allow the model without the
additional input to compensate by \emph{thinking} longer, while charging it
for that additional thinking. In this respect, we follow the
\emph{value-of-computation} perspective of Halpern and
Pass~\cite{HP11,HP14} and the simulation paradigm underlying
\emph{zero-knowledge proofs}, which asks what can be efficiently generated
without access to the information in question~\cite{GMR}.

Said plainly, we seek a notion of value that measures how much the prompt
``helped'' the LLM produce the artifact, while taking computation into account.
In this work, we introduce such a measure of prompt value.

\subsection{From Kolmogorov complexity to an LLM-relative measure}
Classical algorithmic information theory provides a natural starting point.
Fix a universal Turing machine \(U\), and let \(K_U(x)\) denote the
Kolmogorov complexity of \(x\)
\cite{Solomonoff64,Kolmogorov65,Chaitin66}: the length of the shortest program
\(\pi\) that generates \(x\) (i.e., \(U(\pi)=x\)).  Define
\(K_U(x\mid y)\) analogously, with \(y\) supplied as auxiliary input (in other words, the shortest program that generates $x$ given $y$).
Following Kolmogorov, the \emph{algorithmic information}
\cite{ZL70,LiVitanyi} that \(p\) provides about \(z\) is
\begin{equation}\label{eq:classical}
I_U(p:z) \;:=\; K_U(z)-K_U(z\mid p).
\end{equation}
Thus, \(I_U(p:z)\) is the number of bits saved in describing \(z\) once \(p\)
is available, making it a natural measure of the value of the prompt \(p\) for
\(z\).

Two difficulties prevent direct use of this measure. First, Kolmogorov
complexity is uncomputable.
Second, it ignores the computational complexity of the program $\pi$, which makes it unsuitable for capturing ``computational gains".
The standard resource-bounded response to the second difficulty is time-bounded
Kolmogorov complexity \cite{Kolmogorov65,Ko86,Har83,Sipser83}: For a time limit $T$,
\begin{equation}
K^T_U(z \mid p) \;:=\; \min_{\substack{\pi:\; U(\pi,p)=z \\ T_U(\pi,p)\le T}} |\pi|.
\end{equation}
Levin's $Kt$ complexity \cite{Levin73} combines description length and
execution time into one quantity:
\begin{equation}
Kt_U(z \mid p) \;:=\; \min_{\pi:\; U(\pi,p)=z}
\bigl\{ |\pi| + \log_2 T_U(\pi,p) \bigr\}.
\end{equation}
Unfortunately, resource bounds do not make these notions easy to compute (although computable):
Allender et al.\ give worst-case hardness results for resource-bounded
Kolmogorov complexity \cite{ABKMR06}. Liu and Pass connect average-case
hardness of time-bounded Kolmogorov complexity to one-way functions, including
(conditional) equivalences on samplable distributions \cite{LP20,LP21,LP23}.
These results indicate that resource-bounding alone does not make the notion
efficiently computable.

\paragraph{Dealing with Non-thinking LLMs: LLM-Relative Kolmogorov Complexity}
As a warm-up, we first provide a notion of prompt value for \emph{LLMs without
thinking}.  By this we mean an autoregressive model that, given a context \(y\),
directly generates its output one token at a time: each token is sampled from a
distribution determined by \(y\) and the previously generated tokens, and
generation ends when the model emits a distinguished end-of-sequence token $(\EOS)$. In other words, there is no separate ``thinking stage" preceding the output.

Our first step is to make the above-mentioned algorithmic-mutual-information approach efficiently computable by
fixing the reference machine to be the actual LLM production process. When the
question is what an input contributed to \emph{this} process, the process
itself is the natural reference machine. Given a fixed LLM, a ``program'' is a
specification of the model's sampling randomness, which fully determines its
output; we specify the randomness in binary as a real number in $[0,1)$, and
the length of a program is the length of the shortest prefix of this number
that forces the output. Write $K_M(x\mid y)$ for the length of the shortest
such program forcing $x$ in context (i.e., given prompt) $y$, and define
$\Val_M(p;z) := K_M(z) - K_M(z \mid p)$ as the LLM-relative mutual
information.

Alongside it, we use an LLM-relative notion of \emph{a-priori} complexity, the
classical companion of program length in algorithmic information theory
\cite{ZL70,LiVitanyi}:
$\widetilde K_M(x\mid y):=-\log_2P_M(x\mid y)$,
where $P_M(x\mid y)$ denotes the probability that $M$ outputs $x$ given the
prompt $y$, and define the value
$\widetilde{\Val}_M(p;z):=\widetilde K_M(z)-\widetilde K_M(z\mid p)$ as the a-priori
LLM-relative analog of mutual information. As we show in
Section~\ref{sec:exact},
$0\le K_M(x\mid y)-\widetilde K_M(x\mid y)<2$; hence the a-priori analog closely
approximates the program-based notion of LLM-relative mutual information:
$\bigl| \Val_M(p;z) - \widetilde{\Val}_M(p;z) \bigr| < 2$.

Moreover, by definition,
\[
  \widetilde{\Val}_M(p;z)
  =
  \log_2\frac{P_M(z\mid p)}{P_M(z)},
\]
which can be computed directly by summing the base-two log-ratios of the
model's prompted and unprompted next-token probabilities along \(z\). In essence,
\(\widetilde{\Val}_M(p;z)\) has exactly the algebraic form of
\emph{pointwise mutual information} \cite{Fano61,CH90}; this recovers a
non-normalized variant of the ``author-contribution'' score of
\cite{Xie26}.  Thus, a score of \(b\) means that the prompt makes \(z\)
\(2^b\) times as likely to be output by the LLM as without the additional
prompt.  Achille and Soatto~\cite{AS26} also use this log-likelihood ratio to
measure the value of a prompt for producing a specified sequence;
see Section~\ref{sec:related} for a comparison.

For these reasons, we take the a-priori notion as the basis for our
notion of prompt value.

\paragraph{The problem with thinking.}
Valuing prompts for an LLM with thinking is more complicated.  We model such
an LLM as a two-stage autoregressive process.  Given a context \(y\), the model
first generates a finite string \(H^y\) of thinking tokens and then emits a
distinguished end-of-thinking token \(\EOT\), which marks the transition to the
output stage.  The model then generates its output and terminates it with the
end-of-sequence token \(\EOS\). 
For any finite thinking string \(H\),
write
\[
  G_y(z\mid H)
  :=
  P_M(z\mid y\,H\,\EOT)
\]
for the probability that the output stage emits the artifact \(z\)
followed by \(\EOS\).

Thinking changes the valuation problem in two ways. First, the direct
calculation above no longer yields the overall probability that \(M\) produces
\(z\), because this probability must also average over the potentially long,
random thinking route $H$:  in context \(y\), it is
\[
  \E_{H^y}\!\left[G_y(z\mid H^y)\right].
\]
Such an approach was recently explored by Achille and Soatto~\cite{AS26} in the context of verifiable tasks. Although this yields a natural notion of prompt value, it is not the approach
we consider here: it averages together realized thoughts and can be
substantially influenced by rare, unusually successful routes.\footnote{The
marginal probability can be estimated without waiting for a rollout to emit
\(z\), by averaging the conditional probabilities
\(G_y(z\mid H^y)\) over sampled thoughts.  Each conditional probability can
itself be computed directly from the model's next-token probabilities along
\(z\).  Nevertheless, obtaining a reliable multiplicative estimate when the
marginal is extremely small may require a correspondingly large number of
samples.}
We instead view the realized thinking as the random tape of the program,
measure the conditional difficulty of producing \(z\) relative to that tape,
and then take the \emph{median} over the model's thinking randomness.  The
resulting quantity captures the difficulty on a \emph{typical} realized thought
and is insensitive to the one thought in a billion that happens to stumble
upon the decisive idea.

A second issue is that comparing the two contexts at a common truncation index $t$
misses the principal way many inputs help: a good hint does not merely make the
answer more probable at fixed effort, but may remove or lessen the need to think.
To capture this computation saving, the prompted and unprompted processes must
therefore be allowed to use different amounts of thinking---in particular, the
unprompted baseline may need to think longer---and each must be charged for the
computation it uses. The resulting comparison credits a prompt both when it
increases the conditional probability of the output and when it reduces the
computation needed to obtain it.

\paragraph{Our approach: realized-thought Levin complexity.}
Roughly speaking, we solve the first problem by viewing thinking not as part of the description of the program, but rather as \emph{randomness} for the program; that is, we consider a \emph{probabilistic} notion of Kolmogorov complexity in the spirit of \cite{GKLO22}, though defined somewhat differently. We solve the second one by considering Levin's notion of $Kt$-complexity \cite{Levin73} (and thus ``charging'' logarithmically for the computation): in essence, we consider an \emph{LLM-relative notion of probabilistic (a-priori) Levin-Kolmogorov complexity $pKt$}, and our measure is simply algorithmic mutual information defined w.r.t. it.

Concretely, fix the artifact $z$---in the motivating scenario, the output of
a prompted run. 
Sample a thinking \emph{rollout}: let the model think freely until it emits a stop token,
writing $H^y$ for the realized thinking and $H^y_{\le t}$ for the first $t$ tokens; think of this as sampling randomness for the program. Next, note that summing the negative base-two logarithms of the model's next-token probabilities along $z$, including the probability of termination after $z$, computes exactly the
a-priori complexity
$\widetilde K_M\bigl(z \mid y\, H^y_{\le t}\, \EOT\bigr)
= -\log_2 P_M\bigl(z \mid y\, H^y_{\le t}\, \EOT\bigr)$
of the artifact for a model that has already thought $H^y_{\le t}$ (and recall that a-priori complexity corresponds to our notion of program length). We next charge thinking through an externally specified ``token-equivalent'' cost function $\kappa(t)$ and define the
\emph{realized-thought Levin complexity} of $z$ in context $y$ along the
thinking route $H^y$,
\begin{equation}\label{eq:introkt}
\widetilde{Kt}^{\kappa}_M\bigl(z \mid y;\, H^y\bigr)
\;:=\;
\min_{t \in \mathbb N_0}
\Bigl\{
\underbrace{\widetilde K_M\bigl(z \mid y\, H^y_{\le t}\, \EOT\bigr)}_{\text{``description length"}}
\;+\;
\underbrace{\log_2 \kappa(t)}_{\text{log running time}}
\Bigr\} ,
\end{equation}
(i.e., we employ Levin's combination of description length and log running
time, in analogy with the Levin's $Kt$ notion, $Kt(x) = \min_t \{K^t(x) + \log_2 t\}$).

The above quantity is random through the sampled rollout \(H^y\). We summarize
this randomness by taking the median over the model's thinking, defining an
\emph{LLM-relative notion of probabilistic (a-priori) Levin--Kolmogorov complexity} (a la \cite{GKLO22,HLO25}):  
\begin{equation}\label{eq:intropkt}
\pKtt^{\kappa}_M\bigl(z \mid y\bigr)
\;:=\;
\med\Bigl[\,\widetilde{Kt}^{\kappa}_M
\bigl(z \mid y;\,H^y\bigr)\,\Bigr].
\end{equation}
The \emph{prompt value} is then simply algorithmic
mutual information with respect to it,
\begin{equation}\label{eq:introvalue}
\widetilde{\Val}^{\kappa}_M(p; z)
\;:=\;
\pKtt^{\kappa}_M(z)
\;-\;
\pKtt^{\kappa}_M(z \mid p) ,
\end{equation}
in exact analogy with the LLM-relative mutual information
$\widetilde{\Val}_M = \widetilde K_M(z) - \widetilde K_M(z \mid p)$ of the
no-thinking case, and with $\pKtt$ in place of $\widetilde K$ as the notion of
description length.

Note that a prompt is now credited both when it (1) makes the
artifact more probable given the thinking realized and (2) when it eliminates
thinking that the unprompted side must otherwise pay for; and the unprompted
side can compensate for a missing hint by thinking longer, at a price.

As we show, this measure is efficiently estimable in the following
sense. For any sampled thinking route, the realized-thought Levin complexity is
computed exactly from the model's next-token probabilities at each
truncation of the realized thinking: 
although the minimization in~\eqref{eq:ktdef} ranges over all of
\(\mathbb N_0\), only \(t=0,\ldots,S\) need be evaluated for a thinking route of
length \(S\), since every \(t>S\) is weakly dominated by \(t=S\). Moreover, with
$O(\zeta^{-2}\log(1/\eta))$ independent rollouts per side, the empirical
$\delta$-quantile lies, with probability at least $1-\eta$, between the
$(\delta-\zeta)$- and $(\delta+\zeta)$-quantiles of the true distribution.
Applying this guarantee with and without the prompt yields corresponding bounds on the prompt-value estimate.

\paragraph{An economic interpretation of prompt value: token-cost savings.}
We finally consider an economic notion of token cost for generating an artifact
$z$ given a prompt $y$. We first define the cost of reproducing an artifact $z$
given an input $y$ and a realized thinking route $H$ by the expenditure of an
experiment: sample repeated independent attempts, conditional on $y$ and $H$,
until $z$ is reproduced. Each attempt is assigned the declared
thinking-token-equivalent charge $\kappa(|H|)$, so the expected cost of the
repeated sampling process is
$\TC_y(z; H) := \E[N\kappa(|H|)]$. 
Writing
$\TC^*_y(z; H^y) := \min_{0\le t\le |H^y|} \TC_y(z; H^y_{\le t})$ for the
cost at the best prefix of the realized thinking, we show that
\[
2^{\,\pKtt^{\kappa}_M(z \mid y)} \;=\; \med\bigl[\, \TC^*_y(z; H^y) \,\bigr];
\]
that is, the $pKt$-complexity notion is  the logarithm of the \emph{typical} token
expenditure of reproducing $z$ in context $y$. Prompt value is then, by
definition, a ratio of such costs,
\[
2^{\,\widetilde{\Val}^{\kappa}_{M}(p;z)}
\;=\;
\frac{\med_{}\bigl[\TC^*_\eps(z; H^\eps)\bigr]}
     {\med_{}\bigl[\TC^*_p(z; H^p)\bigr]} ,
\]
so a prompt value of $b$ means that reproducing the artifact without the prompt
typically costs $2^b$ times more tokens than with it.

\paragraph{Medians, $\delta$-quantiles, and why not expectations?}
The median reports the typical realization, insensitive to the one thought in a
billion that stumbles onto the decisive idea; thus, the median is arguably a more relevant statistic than expectation in our context. Using a quantile rather than an expectation is also important for the cost interpretation: quantiles commute with exponentiation, whereas expectations do not.

Nothing hinges on the median in particular: every
$\delta$-quantile yields a complexity notion satisfying the same cost identity, and
we more generally define $\pKtt^{\kappa}_{M,\delta}$ by replacing the median by the $\delta$-quantile of the distribution.

\paragraph{What counts as the ``artifact''.}
Our measure is defined relative to a declared artifact \(z\), and the choice of
what counts as the artifact is a substantive modeling decision. Otherwise, it
is easy to manufacture artifacts and prompts having large prompt value. For
instance, let \(z\) be a string of \(n\) tokens drawn uniformly at random, and
let \(p\) supply \(z\) in a form that causes the model to reproduce it reliably.
For a typical such string,
\(\widetilde K_M(z)\) is at least
\(n\log_2|\Sigma|-O(\log n)\) bits with high probability, whereas
\(\widetilde K_M(z\mid p)\) is small. The prompt therefore receives nearly the
full description length of the target even though the target is merely a random string.

The same construction can be attached to a genuine artifact. Suppose that an
LLM, without a prompt, produces a proof \(z^\star\) of a conjecture of
Erd\H{o}s. Let \(r\) be a random string and declare the scored artifact to be
\(z=z^\star\mathbin{\|}r\), with \(r\) embedded as semantically inert text so
that \(z\) remains a valid presentation of the proof. Now supply \(r\) in a
form that causes the model to reproduce it reliably. Without the prompt, the
random suffix typically contributes roughly
\(|r|\log_2|\Sigma|\) bits of description length; once it is supplied,
reproducing it is cheap. The prompt therefore collects nearly the full value of
the suffix even though it played no role in producing the mathematical content
\(z^\star\).

The measure is nevertheless behaving as intended: it prices the description of
the declared artifact. The burden therefore falls on the declaration of what
constitutes that artifact.

When possible, a simple remedy is to take \(z\) to be a canonical
representation of the produced object rather than the particular string the
model happened to emit. When a verifier for the ``artifact class" is available, this is immediate: let
\(z\) be the verifier's canonical verdict, and require the object itself to appear in the realized thinking. Semantically inert padding then cannot appear in \(z\);
see Section~\ref{sec:canonical}.

When no suitable canonicalization or verifier is available, one possible
approach is \emph{semantic re-randomization}: apply a declared,
prompt-independent rewriting procedure (e.g., use an LLM) that preserves the artifact's semantic
content while varying its surface form, and evaluate the prompt after this
transformation. The hope is that semantically inert padding will not survive
such rewriting, so that a prompt supplying only noise receives no value,
whereas a prompt supplying a substantive idea remains useful across different
renderings. We leave the formalization and evaluation of semantic
re-randomization for future work; throughout this paper, we simply assume that
the artifact is specified exogenously.

\paragraph{Relation to prior work.}
A recent work by Xie et al.\ \cite{Xie26} introduces a measure of human
contribution in AI-assisted content generation and evaluates it
experimentally. Our no-thinking value is an unnormalized version of their
score (they additionally divide by the output's self-information); our
treatment thus provides an algorithmic-information theoretic foundation
of that numerator. We emphasize that their score contains no explicit
charge for thinking or computation; dealing with thinking is our main
contribution.

Achille and Soatto~\cite{AS26} consider the same unnormalized
score of Xie et al., applied to a particular sequence of thinking tokens.
They call this quantity algorithmic mutual information relative to the
model and interpret a score of $b$ bits as a $2^b$-fold reduction in the
cost of finding that sequence.

Achille and Soatto~\cite{AS26} also study the computational cost of LLM search. For verifiable tasks, they consider the quantity \(t/F_y(t)\), where \(F_y(t)\) is the probability of finding a verified answer within \(t\) computation steps, and minimize this quantity over \(t\).
One could define prompt value by comparing these search costs with and without the prompt. We do not adopt this approach: as discussed above, estimating it can require exponentially many samples when success depends on exponentially rare thinking routes.
The end of Section~\ref{sec:tokencost} connects this approach to our definitions: In the verified setting, adding a charge for rarity and optimizing over quantile levels recovers the logarithm of their search-cost measure; we thank Achille \cite{achille-personal} for pointing out this connection. 

\paragraph{Paper outline.}
Section~\ref{sec:exact} develops our no-thinking notion and provides its algorithmic foundation.
Section~\ref{sec:thinking} introduces the thinking process, realized-thought Levin complexity,  our probabilistic Levin complexity notion
$\pKtt^{\kappa}_M$ and our final notion of prompt value.
Section~\ref{sec:protocol} gives the estimation protocol and its guarantee.
Section~\ref{sec:tokencost} supplies the economics, defining the cost of
reproducing an artifact by the expenditure of an experiment and showing that
$2^{\pKtt}$ is exactly the typical such cost, so that the prompt value is a ratio
of typical costs.
Section~\ref{sec:experiment} reports a small experiment on GSM8K, a benchmark dataset of grade-school mathematics word problems with step-by-step reference solutions, and Section~\ref{sec:related} discusses related work.

\paragraph{Use of AI.}
While the ideas are my own, large language models, principally ChatGPT and Claude, were extensively used
in drafting, revising, and editing this manuscript (including in expanding proof
sketches into complete proofs). Claude and ChatGPT also implemented the experiments and produced the plots (based on my directions). I take full responsibility for any errors and oversights.

\section{Prompt Value for Non-Thinking LLMs: A Warm-Up}\label{sec:exact}

This section develops a notion of prompt value for non-thinking LLMs; it serves as a warm-up for our actual notion that deals also with LLMs with thinking. We start by defining a notion of a (non-thinking) LLM:

\begin{definition}[Autoregressive LLM]\label{def:llm}
Fix a finite token alphabet $\Sigma$ and a distinguished end-of-sequence token
$\EOS \notin \Sigma$, and let $\Gamma := \Sigma \cup \{\EOS\}$. An
\emph{autoregressive LLM} $M$ specifies, for every context $c \in \Sigma^*$, a
probability distribution $P_M(\cdot \mid c) \in \Delta(\Gamma)$ over the next
token. A finite output is a string $x = (x_1,\dots,x_n) \in \Sigma^*$, generated
by emitting $x_1,\dots,x_n$ and then $\EOS$; its probability in context
$y \in \Sigma^*$ is
\begin{equation}\label{eq:seqprob}
P_M(x \mid y) := \Bigl( \prod_{i=1}^n P_M(x_i \mid y x_{<i}) \Bigr)
P_M(\EOS \mid y x),
\qquad x_{<i} := (x_1,\dots,x_{i-1}).
\end{equation}
We assume \emph{proper termination}: for every context $y$, repeated next-token
sampling emits $\EOS$ almost surely, with a finite expected number of tokens
before termination.
\end{definition}

Any fixed background context---chat template, task preamble,
or deployment-specific instruction---may be absorbed into the definition of
the reference model \(M\).  Write \(\eps\) for the empty additional prompt and
abbreviate $P_M(\cdot):=P_M(\cdot\mid\eps)$.
Thus \(P_M(\cdot)\) denotes the model's output distribution under the fixed
background with no additional prompt, whereas \(P_M(\cdot\mid p)\) denotes the
distribution of the same model under the same background after receiving the
additional prompt \(p\).  We use the same convention for all complexity
quantities below: an omitted conditioning context means conditioning on
\(\eps\).  

\paragraph{An LLM-relative notion of Kolmogorov complexity.}
In ordinary Kolmogorov complexity, a universal Turing machine is the machine
that interprets and executes programs.  Here we instead use the LLM (together
with its sampling procedure) as the underlying machine.  A program specifies
the randomness used to sample from the LLM and thereby determines its
execution.  We represent the complete sampling randomness by a real number
\(\omega\in[0,1)\), written in binary; a finite binary program specifies an
initial segment of that binary expansion, and hence a dyadic subinterval of
possible values of \(\omega\).

\begin{definition}[Output intervals and binary programs]\label{def:programs}
Fix a total order on $\Gamma$. For an output $x \in \Sigma^*$ write
$\bar x := x\,\EOS$, and order outputs by the lexicographic order on the strings
$\bar x$ induced by the order on $\Gamma$. The \emph{output interval} of $x$ in
context $y$ is
\begin{equation}\label{eq:interval}
I_y(x) := \bigl[\, F_M(x \mid y),\; F_M(x \mid y) + P_M(x \mid y) \,\bigr),
\qquad
F_M(x \mid y) := \!\!\sum_{x' \,:\, \bar{x}' < \bar x}\!\! P_M(x' \mid y).
\end{equation}
A \emph{binary program} is a finite string $\pi \in \{0,1\}^*$; with
$m := |\pi|$ and $N(\pi) := \sum_{j=1}^m \pi_j 2^{m-j}$, it names the dyadic
interval
\begin{equation}\label{eq:dyadic}
D_\pi := \Bigl[ \frac{N(\pi)}{2^m}, \frac{N(\pi)+1}{2^m} \Bigr),
\end{equation}
so that $D_\eps = [0,1)$ for the empty program and $|D_\pi| = 2^{-|\pi|}$.
A binary program $\pi$ \emph{forces} output $x$ under context $y$ if
$D_\pi \subseteq I_y(x)$. 
\end{definition}

A binary program can be evaluated efficiently whenever
\(M\)'s next-token distributions are efficiently computable.  Given
\(\pi\in\{0,1\}^*\), let
\[
  \omega_\pi:=\frac{N(\pi)}{2^{|\pi|}}
\]
be the left endpoint of \(D_\pi\).  Starting from the interval \([0,1)\), run
the model's sampler deterministically using \(\omega_\pi\) as its sampling
randomness.  After each generated prefix, compute \(M\)'s next-token
distribution, partition the current interval in the fixed token order into
consecutive left-closed, right-open subintervals having the corresponding
relative lengths, emit the token whose subinterval contains \(\omega_\pi\),
and continue with that subinterval until \(\EOS\) is emitted.  (In essence, this is an instance of the standard arithmetic-decoding procedure~\cite{WNC87}.)

If \(D_\pi\subseteq I_y(x)\), then \(\omega_\pi\in I_y(x)\), so the procedure
recovers \(x\).  It uses one next-token-distribution evaluation per emitted
token.  Consequently, whenever \(M\)'s next-token distributions are
polynomial-time computable, \(x\) can be recovered from \(M\), \(y\), and
\(\pi\) in time polynomial in \(|y|+|\pi|+|x|\).  Thus our LLM-relative
programs are efficiently evaluable whenever \(M\)'s next-token distributions
are.

We are now ready to state the notion of LLM-relative K-complexity, and the notion of algorithmic prompt value for non-thinking LLMs:
\begin{definition}[LLM-relative K-complexity and non-thinking prompt value]\label{def:KM}
Define
\begin{equation}
K_M(x \mid y) := \min\{ |\pi| : D_\pi \subseteq I_y(x) \},
\end{equation}
with value $\infty$ if no such $\pi$ exists. For a prompt $p \in \Sigma^*$ and
an artifact string $z \in \Sigma^*$, define the \emph{non-thinking program-based prompt value}
\begin{equation}
\Val_M(p;z) := K_M(z) - K_M(z \mid p).
\end{equation}
\end{definition}
That is, the (program-based) prompt value is simply the notion of algorithmic mutual information using LLM-relative K-complexity.

Given access to \(M\)'s full next-token distributions,
\(K_M(x\mid y)\) is computable in polynomial time.\footnote{First
compute \(I_y(x)\) by successive interval refinement, the standard construction
underlying arithmetic coding~\cite{WNC87}: Starting from
\(J_{\eps}=[0,1)\), process the tokens of \(x\,\EOS\) in order.  After a prefix
\(u\), partition \(J_u\), in the fixed order on \(\Gamma\), into subintervals
whose relative lengths are the probabilities \(P_M(a\mid yu)\), and let
\(J_{ua}\) be the subinterval corresponding to the next token \(a\).  The
interval obtained after processing \(x\,\EOS\) is exactly \(I_y(x)\).   
Now write \(I_y(x)=[a,b)\).  For each \(m\in\mathbb N_0\), let
$k_m:=\left\lceil 2^m a\right\rceil$.
Note that the interval
  $\left[k_m2^{-m},(k_m+1)2^{-m}\right)$
is the leftmost dyadic interval of length \(2^{-m}\) whose left endpoint is at
least \(a\).  Therefore, some dyadic interval of length \(2^{-m}\) is
contained in \([a,b)\) if and only if $(k_m+1)2^{-m}\leq b$.
Let $m^\star
  :=
  \min\left\{
    m\in\mathbb N_0:
    (k_m+1)2^{-m}\leq b
  \right\}$.
Then \(K_M(x\mid y)=m^\star\), and the \(m^\star\)-bit representation of
\(k_{m^\star}\) is a shortest program forcing \(x\).
} and hence so is program-based prompt value as well.  This
computation, however, requires the model's full next-token distributions
\emph{without approximation}, which are typically unavailable through an LLM
API.
To address this issue, we consider a common variant of Kolmogorov complexity,
\emph{a-priori} complexity~\cite{ZL70,LiVitanyi}: the log-measure of the set
of sampling randomness that produces the output.  We again provide an
LLM-relative notion of this.

\begin{definition}[A priori LLM-relative complexity and measured value]\label{def:measured}
For a context $y$ and string $z$, define
\begin{equation}
\widetilde{K}_M(z \mid y) := -\log_2 P_M(z \mid y),
\end{equation}
with $\widetilde{K}_M(z \mid y) := \infty$ if $P_M(z \mid y) = 0$. For a prompt
$p$ and string $z$, the \emph{non-thinking prompt value} is
\begin{equation}
\widetilde{\Val}_M(p;z) := \widetilde{K}_M(z) - \widetilde{K}_M(z \mid p).
\end{equation}
\end{definition}

By expanding out the definition, we directly get:
\begin{corollary}[Likelihood-ratio form]\label{cor:ratio}
If $z$ has positive probability under both $\epsilon$ and $p$, then for
$z = (z_1,\dots,z_n)$,
\begin{equation}\label{eq:tokenwise}
\widetilde{\Val}_M(p;z) = \log_2 \frac{P_M(z \mid p)}{P_M(z)}
= \sum_{i=1}^n
\log_2 \frac{P_M(z_i \mid p\, z_{<i})}{P_M(z_i \mid z_{<i})}
+ \log_2 \frac{P_M(\EOS \mid p\, z)}{P_M(\EOS \mid z)}.
\end{equation}
\end{corollary}

Note that $\widetilde{\Val}_M$ has exactly the algebraic form of the
pointwise-mutual-information (PMI) functional \cite{Fano61,CH90}, applied to
the model's prompted and unprompted distributions, and can be thought of as a
non-normalized version of the score of \cite{Xie26}. Achille and Soatto~\cite{AS26} also use this log-likelihood ratio to measure the value of a prompt for producing a specified sequence; see Section~\ref{sec:related} for a comparison.

The main result of this section is that the two notions of LLM-relative K-complexity (program-based and a-priori-based) and thus also the two notions of prompt value, coincide up to an additive
constant, much like standard a-priori complexity and (prefix) Kolmogorov
complexity \cite{ZL70}.

\begin{theorem}[Algorithmic semantics of prompt value]\label{thm:semantics}
\leavevmode
\begin{enumerate}
\item[(i)] For every context $y$ and string $z$ with $P_M(z \mid y) > 0$,
\begin{equation}\label{eq:codinggap}
0 \;\le\; K_M(z \mid y) - \widetilde{K}_M(z \mid y) \;<\; 2.
\end{equation}
\item[(ii)] Consequently, for every prompt $p$ and exact string $z$ in the
support of both distributions,
\begin{equation}\label{eq:threebits}
\bigl| \Val_M(p;z) - \widetilde{\Val}_M(p;z) \bigr| < 2.
\end{equation}
\end{enumerate}
\end{theorem}

The formal proof is standard and given in Appendix for completeness~\ref{app:dyadic}. 
For a proof sketch, recall that we associate with an output \(z\) an interval
\(I_y(z)\subseteq[0,1)\) whose length is \(P_M(z\mid y)\), and that a binary
program \(\pi\) forces \(z\) when its dyadic interval \(D_\pi\), of length
\(2^{-|\pi|}\), lies entirely inside \(I_y(z)\). The lower bound in part~(i)
follows because \(D_\pi\) cannot be longer than \(I_y(z)\). For the upper
bound, the proof rounds the left endpoint of \(I_y(z)\) to a dyadic grid and
its length down to an inverse power of two; the resulting two factor-\(2\)
losses suffice. (This upper bound is the standard dyadic-interval argument from arithmetic coding; the resulting upper bound is stated, for example, in~\cite[Section~9.1, p.~102]{PW16}.) Part~(ii) then follows by taking differences.

\begin{remark}[Differences with standard algorithmic mutual information]\label{rem:leverage}
Two disanalogies with universal-machine algorithmic information are
noteworthy. First, non-thinking prompt value can be (very) negative: the conditional
distribution is the model's actual behavior under the prompt, and the model
cannot be assumed to ignore a misleading input at constant cost. Second, a
very short input can carry a very large value: a one-token trigger may raise
the probability of a long output by hundreds of bits. This second phenomenon
is similar to what happens with time-bounded Kolmogorov complexity, where
getting a short key may unlock a long encrypted message \cite{LM93}.
\end{remark}

We observe that \(\widetilde K_M\), and hence non-thinking prompt value, can be efficiently computed from the model's next-token log probabilities along \(z\,\EOS\). Moreover, in contrast to $K_M$, \emph{approximate} log probabilities suffice to approximate these quantities: if every next-token log probability is known within additive error \(\eta\), then \(\widetilde K_M(z\mid y)\) is determined within additive error \(|z\,\EOS|\eta\), and the resulting prompt value within additive error \(2|z\,\EOS|\eta\).

\begin{theorem}[Efficient evaluation of non-thinking prompt value]
\label{thm:efficient}
Let \(z=(z_1,\ldots,z_n)\) have positive probability under both \(\eps\) and
\(p\).  Given exact next-token log probabilities,
\(\widetilde{\Val}_M(p;z)\) can be computed exactly using \(2(n+1)\) queries
and \(O(n)\) arithmetic operations.  More generally, if each queried log
probability is approximated within additive error \(\rho\), the same computation
approximates \(\widetilde{\Val}_M(p;z)\) within \(2(n+1)\rho\).
\end{theorem}

\begin{proof}
By \eqref{eq:tokenwise}, \(\widetilde{\Val}_M(p;z)\) is a signed sum of the
prompted and unprompted log probabilities of the \(n\) tokens of \(z\) and the
final \(\EOS\), giving \(2(n+1)\) terms.  The exact claim follows by summing
them, and the approximation claim follows from the triangle inequality.
\end{proof}

In the sequel, we take  $\widetilde{\Val}_M$ as the starting point for our notion of prompt value for LLMs with thinking,
with Theorem~\ref{thm:semantics} guaranteeing that the reported quantities
track the algorithmic (program-based) values within two bits.

\section{Thinking and the prompt value}\label{sec:thinking}

We move on to consider LLMs with thinking.

\subsection{LLMs with thinking}\label{sec:process}
We start by formalizing a thinking LLM. Informally, an LLM with thinking operates in two stages. It first generates a
thinking route \(H\), stopping when it emits a distinguished
end-of-thinking token \(\EOT\) (or \(\EOS\)). It then generates its output
conditioned on the realized route \(H\).

\begin{definition}[LLM with thinking]\label{def:thinking}
An \emph{LLM with thinking} consists of an autoregressive LLM \(M\), as in
Definition~\ref{def:llm}, together with a distinguished end-of-thinking token
\(\EOT\in\Sigma\). In context \(y\), the thinking stage samples tokens
autoregressively from \(M\) until either \(\EOT\) or \(\EOS\) is emitted. Let
\(H^y\) denote the tokens generated before this stop token, which is not
included in \(H^y\); we call one execution of the thinking-stage sampler a \emph{rollout}, and call
the resulting token sequence \(H^y\) its \emph{thinking route}.

After the thinking stage, the output stage runs \(M\) in context
\(y\,H^y\,\EOT\) and generates an output terminating with \(\EOS\). (Thus, even
if the thinking stage stopped at \(\EOS\), the declared two-stage process
proceeds by appending \(\EOT\) and running the output stage.)

Write \(S:=|H^y|\), and, for \(t\in\mathbb N_0\), let \(H^y_{\leq t}\) be the
first \(\min\{t,S\}\) tokens of \(H^y\). For any finite thinking route \(H^y\),
define the output-stage probability
\[
G_y(z\mid H^y):=P_M(z\mid y\,H^y\,\EOT).
\]
\end{definition}

\paragraph{External oracle calls.}
The same treatment extends to thinking processes that interact with external
tools or oracles. In this case, \(H^y\) denotes the realized interaction
transcript, including both the calls and their responses, and its distribution
is induced by the combined LLM--oracle process. We never evaluate the
probability of \(H^y\) itself; at each prefix, we evaluate only the probability
of producing \(z\) conditional on that realized transcript. Thus the oracle
responses need not be generated autoregressively by \(M\).
We require only sampling access, not likelihood access, to the oracle.
The combined LLM--oracle process must be independently restartable across
rollouts.

The only additional ingredient is cost accounting. The token-equivalent time
function introduced below should charge for the oracle calls as well as for the
LLM's computation. Formally, one may allow the cost to depend
nondecreasingly on the realized transcript prefix. Equivalently, one may pad
the accounting timeline of each oracle call with a number of virtual thinking
tokens corresponding to its declared token-equivalent cost. These virtual
tokens are used only for accounting and are not supplied to the model. Under
this interpretation, the definitions and results below apply unchanged.

For simplicity, we restrict the subsequent formal treatment to thinking without external oracle calls.

\subsection{Realized-thought Levin complexity}\label{sec:cost}
We proceed to formalizing an LLM-relative notion of Levin-complexity \cite{Levin73} for LLMs with thinking.
Towards this goal, we start by providing a notion of \emph{realized-thought
Levin complexity}, for a \emph{fixed} (i.e., realized) thinking route $H$. (Looking forward, in Section~\ref{sec:pkt}, we take
$H$ to be a random variable (over thinking) and summarize it by its median, obtaining the
complexity notion $\pKtt^{\kappa}_M$; the prompt value is then algorithmic
mutual information with respect to that notion.)

Charging thinking requires a declared unit of ``running-time"; we count time in
``thinking-token'' equivalent units through an externally specified cost
function:

\begin{definition}[Token-equivalent time]\label{def:kappa}
A \emph{token-equivalent time function} is a family
$\kappa_{y,z} : \mathbb N_0 \to [1,\infty)$, indexed by contexts $y$ and
artifacts $z$, nondecreasing in $t$. When $y$ and $z$ are clear from
context, we write simply \(\kappa(t)\).
\end{definition}

\(\kappa(t)\) is simply the declared token-equivalent cost assigned to a
production that uses \(t\) tokens of thinking in context \(y\). A natural
choice is the \emph{generated-thought} cost
\[
  \kappa_{y,z}^{\mathrm{gen}}(t)
  =
  c_{\mathrm{pre}}|y|
  +
  c_{\mathrm{dec}}\left(t+1+
  \left|z\,\EOS\right|\right).
\]
Here \(c_{\mathrm{pre}}\) is the cost per token processed during prefill,
\(c_{\mathrm{dec}}\) is the cost per sequentially generated token, and the
additional \(1\) accounts for \(\EOT\). This convention charges for prefilling
the supplied context and then sequentially generating the thinking, \(\EOT\),
and the artifact. Because prefill processes many tokens in parallel whereas
decoding is sequential, \(c_{\mathrm{dec}}/c_{\mathrm{pre}}\) can naturally be
much larger than one. (In Section~\ref{sec:experiment} we normalize
\(c_{\mathrm{pre}}=1\), use \(c_{\mathrm{dec}}=32\) as a representative
default, and report sensitivity to this choice.)

Another simple choice is the \emph{prefix-prefill} cost
\[
  \kappa_{y,z}^{\mathrm{pre}}(t)
  =
  c_{\mathrm{pre}}\bigl(|y|+t+1\bigr)
  +
  c_{\mathrm{dec}}
  \left|z\,\EOS\right|.
\]
This convention treats the realized thinking prefix as already available and re-executes it as part of the prefill, rather than charging for generating it
token by token. It is therefore less direct as an accounting of the original
online production of the thought. It nevertheless has an especially natural
token-cost interpretation: Section~\ref{sec:tokencost} shows that it is the
per-attempt cost arising when a fixed realized thought is replayed before fresh
attempts to reproduce the artifact.

We now state the notion of LLM-relative realized-thought Levin complexity: 
\begin{definition}[Realized-thought Levin complexity]\label{def:kt}
Fix a model $M$ and a token-equivalent time function $\kappa$. For a context
$y$, an artifact $z$, and a thinking route $H^y$, define
\begin{equation}\label{eq:ktdef}
\widetilde{Kt}^{\kappa}_M\bigl(z \mid y;\, H^y\bigr)
\;:=\;
\min_{t \in \mathbb N_0}
\Bigl\{
\widetilde K_M\bigl(z \mid y\, H^y_{\le t}\, \EOT\bigr)
\;+\; \log_2 \kappa(t)
\Bigr\} ,
\end{equation}
\end{definition}
In other words, given a fixed thinking route $H^y$, we consider exactly Levin's combination of description length and running time from the notion of $Kt$-complexity \cite{Levin73}: description length plus log running time,
minimized over time. 

Note that although the minimization ranges over all of \(\mathbb N_0\), it
suffices to consider \(t=0,\ldots,S\):
\(H^y_{\leq t}=H^y\) for every \(t\geq S\), so the conditional-complexity
term is constant over these choices, while
\(\kappa(t)\geq\kappa(S)\) by monotonicity; thus, no \(t>S\) can improve
upon \(t=S\), and the minimum can be computed by evaluating the \(S+1\)
truncations \(t=0,\ldots,S\).

A program-based companion replaces a-priori complexity by program length:
\[
Kt^{\kappa}_M\bigl(z \mid y;\, H^y\bigr)
\;:=\;
\min_{t \in \mathbb N_0}
\Bigl\{ K_M\bigl(z \mid y\, H^y_{\le t}\, \EOT\bigr) + \log_2 \kappa(t) \Bigr\}.
\]

\subsection{Probabilistic Levin complexity: a median over random realized thoughts}\label{sec:pkt}
So far, the complexity has been defined relative to a particular realized
thinking route. Viewing the thinking as the random tape of the program (or,
equivalently, representing it by the random bits that cause the LLM to
generate that route), we now define an LLM-relative notion of
\emph{probabilistic Levin--Kolmogorov complexity} in the spirit of
\cite{GKLO22,HLO25}. Namely, we summarize over the model's thinking randomness by
taking a $\delta$-quantile of the resulting realized-thought $Kt$-complexity. The resulting notion is an LLM-relative, a-priori $Kt$ counterpart of probabilistic time-bounded Kolmogorov complexity.
For \(\delta\in(0,1]\) and a random variable \(X\) taking values in
\([0,\infty]\), write
\[
\med_{\delta}[X]
:=
\inf\{a\in[0,\infty]:\Pr[X\leq a]\geq\delta\}
\]
for its lower \(\delta\)-quantile.

\begin{definition}[LLM-relative probabilistic Levin complexity]
\label{def:pkt}
Fix a model \(M\), a token-equivalent time function \(\kappa\), and a level
\(\delta\in(0,1]\). For a context \(y\) and artifact \(z\), the
\emph{probabilistic (a-priori) Levin complexity} of \(z\) given \(y\) at level
\(\delta\) is
\begin{equation}\label{eq:pktdef}
\pKtt^{\kappa}_{M,\delta}\bigl(z\mid y\bigr)
:=
\med_{\delta}\Bigl[
\widetilde{Kt}^{\kappa}_M\bigl(z\mid y;\,H^y\bigr)
\Bigr],
\end{equation}
where the quantile is taken over the distribution of the thinking route
\(H^y\) induced by a random rollout of \(M\) in context \(y\).
The
program-based companion \(\pKt^{\kappa}_{M,\delta}(z\mid y)\) is defined
analogously instead using \(Kt^{\kappa}_M\).
\end{definition}
Whenever $\delta$ is clear from context, we suppress it from the notation; we
typically consider the median $\delta=1/2$. 

\subsection{The prompt value}\label{sec:value}
We define prompt value as  algorithmic mutual information with respect to
$\pKtt^{\kappa}_M$, exactly as $\widetilde{\Val}_M$ was with respect to
$\widetilde K_M$ in Section~\ref{sec:exact}.

\begin{definition}[Prompt value]\label{def:medval}
Fix a model $M$, a token-equivalent time function $\kappa$, a level
$\delta\in(0,1]$, a prompt $p\in\Sigma^*$, and an artifact $z\in\Sigma^*$. Whenever
$\pKtt^{\kappa}_{M,\delta}(z)$ and $\pKtt^{\kappa}_{M,\delta}(z \mid p)$ are
both finite, the \emph{prompt value} of the prompt $p$ for the artifact $z$
is
\begin{equation}\label{eq:valdef}
\widetilde{\Val}^{\kappa}_{M,\delta}(p; z) \;:=\;
\pKtt^{\kappa}_{M,\delta}(z) - \pKtt^{\kappa}_{M,\delta}(z \mid p) ,
\end{equation}
and $\Val^{\kappa}_{M,\delta}(p;z) := \pKt^{\kappa}_{M,\delta}(z)
- \pKt^{\kappa}_{M,\delta}(z \mid p)$ is its program-based companion. 
\end{definition}

A prompt is therefore credited both when it makes the artifact more probable
given the realized thinking and when it eliminates thinking that the unprompted
side must otherwise pay for; and the unprompted side can compensate for a
missing hint by thinking longer, at a price.

As before, the a-priori and program-based version differ by at most 2:
\begin{proposition}[Program semantics of $\pKtt$ and of the prompt value]
\label{prop:pkt-semantics}
Fix $M$, $\kappa$ and $\delta\in(0,1]$. For every context $y$ and artifact $z$, the
program-based and a-priori complexities are finite together, and whenever they are finite:
\[
0 \;\le\; \pKt^{\kappa}_{M,\delta}(z \mid y)
- \pKtt^{\kappa}_{M,\delta}(z \mid y) \;\le\; 2 .
\]
Consequently, whenever the two complexities of Definition~\ref{def:medval} are
finite,
\[
\bigl| \Val^{\kappa}_{M,\delta}(p;z)
- \widetilde{\Val}^{\kappa}_{M,\delta}(p;z) \bigr| \;\le\; 2 .
\]
\end{proposition}

\begin{proof}
For each \(t\in\mathbb N_0\), let
\(c_t:=y\,H^y_{\leq t}\,\EOT\). By
Theorem~\ref{thm:semantics}(i), \(K_M(z\mid c_t)\) and
\(\widetilde K_M(z\mid c_t)\) are finite together and satisfy
\[
\widetilde K_M(z\mid c_t)
\leq K_M(z\mid c_t)
< \widetilde K_M(z\mid c_t)+2.
\]
Adding \(\log_2\kappa(t)\) preserves these inequalities pointwise in \(t\).
Since the relevant minima are attained, minimization preserves the strict
upper bound:
\[
\widetilde{Kt}^{\kappa}_M(z\mid y;H^y)
\leq Kt^{\kappa}_M(z\mid y;H^y)
< \widetilde{Kt}^{\kappa}_M(z\mid y;H^y)+2,
\]
with the two quantities finite together. Taking lower quantiles preserves the
weak inequalities, but may turn the strict upper bound into equality\footnote{More generally, if \(X(\omega)\leq Y(\omega)\) for every outcome
\(\omega\) (or merely almost surely), then
\(\Pr[X\leq a]\geq\Pr[Y\leq a]\) for every \(a\). Hence
\(\{a:\Pr[Y\leq a]\geq\delta\}\subseteq
\{a:\Pr[X\leq a]\geq\delta\}\), and taking infima gives
\(\med_\delta[X]\leq\med_\delta[Y]\).}; hence
\[
\pKtt^{\kappa}_{M,\delta}(z\mid y)
\leq \pKt^{\kappa}_{M,\delta}(z\mid y)
\leq \pKtt^{\kappa}_{M,\delta}(z\mid y)+2,
\]
again with the two quantities finite together. Finally,
\[
\Val^\kappa_{M,\delta}(p;z)
-\widetilde{\Val}^\kappa_{M,\delta}(p;z)
=
\bigl(
\pKt^\kappa_{M,\delta}(z)
-\pKtt^\kappa_{M,\delta}(z)
\bigr)
-
\bigl(
\pKt^\kappa_{M,\delta}(z\mid p)
-\pKtt^\kappa_{M,\delta}(z\mid p)
\bigr).
\]
Each parenthesized term lies in \([0,2]\), so their difference has absolute
value at most \(2\).
\end{proof}
\subsection{Canonical targets and verified acceptance}\label{sec:canonical}
The exact-string requirement on $z$ is a modelling choice about what constitutes
the ``artifact'' $z$ relative to which we are measuring the prompt's value. If
one wishes to price a \emph{class of acceptable artifacts} rather than a string---``the
route contains a proof the verifier accepts''---one declares a machine that
thinks and then emits a canonical verdict, and takes $z$ to be that verdict.
Concretely, let
\(\mathcal V:\Sigma^*\to\{0,1\}\) be an efficiently computable predicate on
the transcript preceding \(\EOT\). Extend the declared machine with a verdict
step that, after the realized transcript \(yH\), emits \(\ACC\) if
\(\mathcal V(yH)=1\) and \(\REJ\) otherwise, and then emits \(\EOS\).

Under this type of ``verified artifacts", the measure takes a particularly simple form. Since the
verdict step is deterministic, $P_M(\ACC \mid y\,H_{\le t}\,\EOT)$ is $1$ on
accepting prefixes and $0$ on rejecting ones, and hence
$\widetilde K_M(\ACC \mid y\,H_{\le t}\,\EOT) \in \{0, \infty\}$. Define the
``acceptance cost" of a route $H$ by
\[
C_y(H):=\min\{\kappa(t):t\in\mathbb N_0,\ \mathcal V(yH_{\le t})=1\},
\qquad \min\emptyset:=\infty.
\]
It follows that
$\widetilde{Kt}^{\kappa}_M(\ACC \mid y;H^y)=\log_2 C_y(H^y)$. Writing $\tau_y$
for the median of $C_y(H^y)$ over rollouts in context $y$, we thus get
\begin{equation}
\pKtt^{\kappa}_M(\ACC \mid y) = \log_2 \tau_y ,
\qquad
\widetilde{\Val}^{\kappa}_M(p; \ACC) = \log_2 \frac{\tau_\eps}{\tau_p} ,
\end{equation}
whenever the two medians are finite.

\section{Estimation}\label{sec:protocol}
We observe that \(\pKtt\) is efficiently estimable in the following sense:
the empirical median obtained from polynomially many rollouts lies, with high
probability, between quantiles arbitrarily close to the population median.
Applying this guarantee with and without the prompt yields corresponding bounds
on the prompt-value estimate.

\begin{center}
\fbox{%
\begin{minipage}{%
  \subonly{0.94\columnwidth}%
  \fullonly{0.94\textwidth}%
}
\subonly{\small}
\textbf{Protocol.}
Fix a model $M$, a token-equivalent time function $\kappa$, a level
$\delta\in(0,1]$, and a number $k\in\mathbb N$ of rollouts per side. 
{\bf Input}: artifact $z$ and prompt $p$.

\begin{enumerate}
\setlength{\itemsep}{2pt}
\setlength{\parskip}{0pt}
\setlength{\parsep}{0pt}

\item For each context \(y\in\{p,\eps\}\), perform \(k\) independent complete
thinking rollouts, producing routes
\(H^{y,(1)},\ldots,H^{y,(k)}\) with respective lengths
\(S^{y,(1)},\ldots,S^{y,(k)}\). For each route \(i\), evaluate every
truncation \(t=0,\ldots,S^{y,(i)}\) and form its routewise minimum
\(\widetilde{Kt}^{\,y,(i)}\) according to \eqref{eq:ktdef}.

\item For each \(y\in\{\eps,p\}\), let \(\widehat m_k^{\,y}\) be the empirical
lower \(\delta\)-quantile\footnote{That is, the
\(\lceil\delta k\rceil\)-th smallest of the \(k\) observed values.} of
\(\widetilde{Kt}^{\,y,(1)},\ldots,\widetilde{Kt}^{\,y,(k)}\), and output
\[
\widehat{\Val}^{\kappa}_{M,\delta,k}(p;z)
:=
\widehat m_k^{\,\eps}-\widehat m_k^{\,p}.
\]
\end{enumerate}
\end{minipage}%
}
\end{center}

\noindent
For \(\delta\in(0,1)\) and
\(|\zeta|<\min\{\delta,1-\delta\}\), define the
\(\zeta\)-offset prompt value by
\[
\widetilde{\Val}^{\kappa}_{M,\delta;\zeta}(p;z)
:=
\pKtt^\kappa_{M,\delta+\zeta}(z)
-
\pKtt^\kappa_{M,\delta-\zeta}(z\mid p).
\]
Thus
\(\widetilde{\Val}^{\kappa}_{M,\delta;0}
=\widetilde{\Val}^{\kappa}_{M,\delta}\), and
\(\widetilde{\Val}^{\kappa}_{M,\delta;\zeta}\) is nondecreasing in \(\zeta\).

\begin{theorem}[Efficient estimation of \(\pKtt\) and prompt value]\label{thm:estimate}
Run the preceding protocol with \(k\in\mathbb N\) and
\(\delta\in(0,1)\). For every
\(\zeta\in(0,\min\{\delta,1-\delta\})\), with probability at least
\(1-4\exp(-2k\zeta^2)\), simultaneously for \(y\in\{p,\eps\}\),
\[
\widehat m_k^{\,y}
\in
\left[
\pKtt^\kappa_{M,\delta-\zeta}(z\mid y),
\pKtt^\kappa_{M,\delta+\zeta}(z\mid y)
\right].
\]
Consequently, with the same probability, whenever the two empirical
quantiles are finite and the interval below is defined,
\[
\widehat{\Val}^{\kappa}_{M,\delta,k}(p;z)
\in
\left[
\widetilde{\Val}^{\kappa}_{M,\delta;-\zeta}(p;z),
\widetilde{\Val}^{\kappa}_{M,\delta;\zeta}(p;z)
\right].
\]
If next-token sampling from \(M\), evaluation of the artifact probabilities,
and evaluation and comparison of \(\kappa_{y,z}(t)\) can be performed in time
polynomial in their input lengths, then the protocol runs in time polynomial
in \(k\), \(T\), \(|p|\), and \(|z|\), where \(T\) is the maximum
thinking time attained by \(M\) in the protocol on inputs \(p\) and \(\eps\).
\end{theorem}

\begin{proof}
Fix \(y\in\{p,\eps\}\), let
\(X_y:=\widetilde{Kt}^{\kappa}_M(z\mid y;H^y)\), and set
\[
q_-:=\pKtt^\kappa_{M,\delta-\zeta}(z\mid y),
\qquad
q_+:=\pKtt^\kappa_{M,\delta+\zeta}(z\mid y).
\]
The \(k\) routewise values used to compute \(\widehat m_k^{\,y}\) are
independent copies of \(X_y\). By the definition of a lower quantile,\footnote{Recall, for
\(q_\alpha:=\inf\{a:\Pr[X\leq a]\geq\alpha\}\), one has
\(\Pr[X<q_\alpha]\leq\alpha\leq\Pr[X\leq q_\alpha]\).}
\[
\Pr[X_y<q_-]\leq\delta-\zeta,
\qquad
\Pr[X_y\leq q_+]\geq\delta+\zeta.
\]
If \(\widehat m_k^{\,y}<q_-\), then at least
\(\lceil\delta k\rceil\) samples are strictly below \(q_-\), so by the Chernoff bound\footnote{Recall, for independent Bernoulli random variables
\(B_1,\ldots,B_k\), the additive Chernoff bound states
\(\Pr[\sum_i B_i-\E[\sum_i B_i]\geq\zeta k]\leq\exp(-2k\zeta^2)\) and
\(\Pr[\E[\sum_i B_i]-\sum_i B_i\geq\zeta k]\leq
\exp(-2k\zeta^2)\)~\cite{Chernoff52}.}
this event is bounded by \(\exp(-2k\zeta^2)\). Similarly, if
\(\widehat m_k^{\,y}>q_+\), then fewer than
\(\lceil\delta k\rceil\) samples are at most \(q_+\), and Chernoff again bounds this event by \(\exp(-2k\zeta^2)\). Thus the
containment fails for a given \(y\) with probability at most
\(2\exp(-2k\zeta^2)\). A union bound over \(y\in\{p,\eps\}\) gives the stated
probability.

On this simultaneous event, subtracting the interval for \(y=p\) from the
interval for \(y=\eps\) gives
\[
\widehat m_k^{\,\eps}-\widehat m_k^{\,p}
\in
\left[
\pKtt^\kappa_{M,\delta-\zeta}(z)
-\pKtt^\kappa_{M,\delta+\zeta}(z\mid p),
\;
\pKtt^\kappa_{M,\delta+\zeta}(z)
-\pKtt^\kappa_{M,\delta-\zeta}(z\mid p)
\right].
\]
By the definitions of
\(\widehat{\Val}^{\kappa}_{M,\delta,k}\) and
\(\widetilde{\Val}^{\kappa}_{M,\delta;\zeta}\), this is exactly the claimed
prompt-value interval.

For the running time, every sampled thinking route has at most \(T+1\) relevant
truncations. The protocol evaluates the artifact probability and
\(\kappa_{y,z}(t)\) at each truncation and minimizes over them. Suppressing the
fixed background context, every evaluation processes at most
\(|p|+T+1+|z\,\EOS|\) tokens. Under the stated assumptions, these evaluations,
the minimizations, and the two empirical-quantile computations take time
polynomial in the claimed parameters.
\end{proof}

\begin{remark}[Transformer implementation]
For a transformer with an append-only KV cache, the cache for every earlier
thinking prefix is an initial segment of the cache for the complete route, so
that prefix need not be recomputed.
Suppressing the fixed background context, the protocol can
therefore be implemented using
\[
O\!\left(k|p|
+k(T+1)\bigl(1+|z|\bigr)\right)
\]
model-token evaluations.
\end{remark}

\section{Prompt value and token-equivalent reproduction cost}\label{sec:tokencost}

We here consider an economic interpretation of our notions. We first define
the cost of reproducing an artifact through an experiment that, given a
\emph{fixed} thinking route, repeatedly prepares the output stage from that
route and samples it with fresh randomness until it recovers the artifact.
Each attempt is assigned the token-equivalent charge specified by the declared
cost function \(\kappa\).

We then show that $2^{\pKtt^{\kappa}_M(z \mid y)}$ is exactly the
\emph{typical} such cost in context \(y\), where typicality is taken over the distribution of thinking
routes induced by random rollouts, so that exponentiated prompt value is the
ratio of the typical costs without and with the prompt.

\begin{definition}[Reproduction experiment and token cost]\label{def:tc}
Fix a model \(M\) and a token-equivalent time function \(\kappa\). Let \(y\)
be a context, \(z\) an artifact, and \(H\) a thinking string. Write
\(g:=G_y(z\mid H)\). The
\emph{reproduction experiment} for $z$ given $H$
makes independent, identically distributed attempts. Each attempt is charged
$\kappa(|H|)$ token-equivalent units, and a fresh output stage is run in the
context $y\,H\,\EOT$. An attempt
\emph{succeeds} if its output is exactly $z$, which occurs with probability
$g$. The experiment halts at the first success, at index $N$, and the
\emph{token-equivalent reproduction cost} of $z$ given $H$ is its expected
total expenditure
$\TC_y(z; H) := \E[N\kappa(|H|)]$.
For a thinking route $H$ with
truncations $H_{\le t}$, the
\emph{best-prefix reproduction cost} is
$\TC^*_y(z; H) := \min_{0\le t\le |H|} \TC_y(z; H_{\le t})$.
\end{definition}

Thus \(\TC_y(z;H)\) is the expected expenditure conditional on the fixed
route \(H\); the quantiles below are taken over routes sampled from random rollouts.

\begin{theorem}[Reproduction cost]\label{thm:bill}
Fix $M$ and $\kappa$. For every context $y$, artifact $z$, and thinking string
$H$,
$$\displaystyle \TC_y(z; H) = \frac{\kappa(|H|)}{G_y(z \mid H)},$$ with the
right-hand side interpreted as $+\infty$ when $G_y(z\mid H)=0$.
\end{theorem}

\begin{proof}
Fix \(H\) and write \(g:=G_y(z\mid H)\). If \(g=0\), both
\(\TC_y(z;H)\) and \(\kappa(|H|)/g\) are defined to be infinite, so assume
\(g>0\). Because each attempt uses fresh output-stage randomness while keeping
\(H\) fixed, the attempts are independent and each succeeds with probability
\(g\). Thus \(N\) is geometric with parameter \(g\), and
\(\E[N]=1/g\). Since each attempt incurs the charge \(\kappa(|H|)\), the total
charge is \(N\kappa(|H|)\). Therefore
\[
\TC_y(z;H)
=
\E\!\left[N\kappa(|H|)\right]
=
\kappa(|H|)\E[N]
=
\frac{\kappa(|H|)}{g}.
\]
\end{proof}

\begin{proposition}[Realized-thought complexity as reproduction cost]
\label{prop:kttc}
Fix $M$ and $\kappa$. For every context $y$, artifact $z$, and rollout $H^y$,
\begin{equation}\label{eq:kttc}
2^{\,\widetilde{Kt}^{\kappa}_M(z \mid y;\, H^y)} \;=\; \TC^*_y(z; H^y).
\end{equation}
\end{proposition}

\begin{proof}
Let $S=|H^y|$. Comparing Definition~\ref{def:kt} with
Theorem~\ref{thm:bill}, and using that exponentiation is increasing, gives
\begin{align*}
2^{\,\widetilde{Kt}^{\kappa}_M(z\mid y;H^y)}
&=
\min_{0\le t\le S}
2^{\,-\log_2G_y(z\mid H^y_{\le t})+\log_2\kappa(t)}\\
&=
\min_{0\le t\le S}
\frac{\kappa(t)}{G_y(z\mid H^y_{\le t})}\\
&=
\min_{0\le t\le S}\TC_y(z;H^y_{\le t})
=\TC_y^*(z;H^y).
\end{align*}
\end{proof}

\begin{theorem}[$\pKtt$ and prompt value as typical token costs]
\label{thm:pktcost}
Fix \(M\), \(\kappa\), and \(\delta\in(0,1]\). For every context \(y\) and
artifact \(z\),
\[
2^{\,\pKtt^{\kappa}_{M,\delta}(z \mid y)}
\;=\;
\med_{\delta}\bigl[\,\TC^*_y(z;H^y)\,\bigr],
\]
where \(\med_{\delta}\) is taken over the distribution of the thinking route
\(H^y\) generated by a random rollout of \(M\) in context \(y\).
Consequently, every prompt $p$ and artifact $z$ such that both complexities in
Definition~\ref{def:medval} are finite,
\[
2^{\,\widetilde{\Val}^{\kappa}_{M,\delta}(p;z)}
\;=\;
\frac{\med_{\delta}\bigl[\TC^*_\eps(z; H^\eps)\bigr]}
     {\med_{\delta}\bigl[\TC^*_p(z; H^p)\bigr]} .
\]
\end{theorem}

\begin{proof}
Lower quantiles commute with continuous strictly increasing
transformations.\footnote{That is,
$\med_\delta[\varphi(X)]=\varphi(\med_\delta[X])$ for continuous strictly
increasing $\varphi$, with the same identity on the extended real line when
$\varphi(+\infty)=+\infty$.}
Applying this fact to \(\varphi(x)=2^x\) and using
Proposition~\ref{prop:kttc} gives
\[
\begin{aligned}
2^{\,\pKtt^{\kappa}_{M,\delta}(z\mid y)}
&=
2^{\,\med_\delta\!\left[
\widetilde{Kt}^{\kappa}_M(z\mid y;H^y)
\right]}\\
&=
\med_\delta\!\left[
2^{\,\widetilde{Kt}^{\kappa}_M(z\mid y;H^y)}
\right]\\
&=
\med_\delta\bigl[\TC_y^*(z;H^y)\bigr].
\end{aligned}
\]
Exponentiating Definition~\ref{def:medval} and applying this identity to its
two terms gives the ratio.
\end{proof}

The theorem gives prompt value its token-cost meaning: a value of
$\widetilde{\Val}^{\kappa}_M(p;z) = b$ means that reproducing this artifact
without the prompt typically costs $2^{b}$ times more token-equivalent units
than with it. 

\paragraph{On the cost function and more general notions of $\widetilde{Kt}$}
The two cost conventions introduced above admit the following operational
interpretations. Under the
``prefix-prefill" cost convention,
\(\kappa^{\mathrm{pre}}\), the realized thought is treated as already
available: each attempt freshly prefills \(y\,H\,\EOT\) and then runs the
output stage with fresh randomness. This convention is natural when the thought
is already available, and is a good amortized approximation when many
reproduction attempts are expected. Under the ``generated-thought" cost
convention, \(\kappa^{\mathrm{gen}}\), each
attempt instead reruns the thinking stage with the same thinking-stage
randomness, thereby regenerating the same route \(H\) sequentially before
running the output stage with fresh randomness. This convention is natural
when each attempt must reproduce the complete online computation, and is closer
to the total cost when only a few attempts are expected and the cost of
initially generating the thought cannot be amortized.

Both cost conventions are stylized approximations of the
implementation-dependent cost of running an LLM. More generally, Levin's
objective function is only one particular way of combining success probability and
computation. Let
\[
  F:[0,1]\times[1,\infty)\longrightarrow\overline{\mathbb R},
  \qquad F(0,c)=+\infty
\]
be any declared function that is nonincreasing in its first argument and
nondecreasing in its second, and define
\[
  \widetilde{Kt}^{F,\kappa}_M(z\mid y;H^y)
  :=
  \min_{t\in\mathbb N_0}
  F\!\left(
    G_y(z\mid H^y_{\leq t}),
    \kappa(t)
  \right).
\]
The usual Levin objective corresponds to
\[
  F_{\mathrm{Lev}}(g,c)
  =
  -\log_2 g+\log_2 c
  =
  \log_2\frac{c}{g}.
\]
Other choices of \(F\) can encode more detailed production costs, including
different charges for initially generating a thought, replaying it, retaining
a KV cache, or making subsequent output attempts. In particular, if
\(T(g,c)\) denotes the expected token cost of an explicitly specified
production procedure with success probability \(g\) and computational charge
\(c\), one may take \(F(g,c)=\log_2T(g,c)\).\footnote{For example, suppose the
first attempt sequentially generates the fixed thinking route and costs
\(\kappa^{\mathrm{gen}}(t)\), while each subsequent attempt freshly prefills
that route and costs \(\kappa^{\mathrm{pre}}(t)\). If \(N\) is geometric with
success probability \(g\), the expected total cost is
\[
\E\!\left[
\kappa^{\mathrm{gen}}(t)+(N-1)\kappa^{\mathrm{pre}}(t)
\right]
=
\kappa^{\mathrm{gen}}(t)
+\left(\frac1g-1\right)\kappa^{\mathrm{pre}}(t).
\]
Taking the logarithm of this expression gives the corresponding choice of
\(F\).}
We focus on the Levin choice
because it is simple and yields the exact reproduction-cost
interpretation above; richer cost aggregators can be used when a more detailed implementation-specific accounting is desired.

\paragraph{Removing the quantile parameter.}
Besides changing the cost assigned to a realized thinking route,
one may also change how costs are summarized over random thinking
routes. A natural attempt to remove the choice of quantile is to
charge for the probability of obtaining a sufficiently favorable
thinking route. Define
\begin{equation}\label{eq:pkt-single}
\mathrm{pk}^\kappa_M(z\mid y)
:=
\inf_{0<\delta\le1}
\left\{
\pKtt^\kappa_{M,\delta}(z\mid y)+\log_2\frac1\delta
\right\}.
\end{equation}
This balances the cost of using favorable thinking routes against
how rarely they occur. The corresponding prompt value is
\(\mathrm{pk}^\kappa_M(z\mid\eps)-\mathrm{pk}^\kappa_M(z\mid p)\),
whenever both terms are finite.

In the verified setting of Section~\ref{sec:canonical}, this
construction recovers the logarithm of the search-cost measure
considered by Achille and Soatto~\cite{AS26}.
Recall the acceptance cost \(C_y(H^y)\) from Section \ref{sec:canonical}
 and write
\[
F_y(b):=\Pr[C_y(H^y)\le b],
\qquad
q_{y,\delta}:=\med_\delta[C_y(H^y)].
\]
Thus \(F_y(b)\) is the probability of acceptance within budget \(b\),
while \(q_{y,\delta}\) is the smallest budget sufficient for a
fraction \(\delta\) of the thinking routes. Their measure---which
they refer to as proper time---is essentially\footnote{They count
computation steps rather than allowing a general cost function,
corresponding to \(\kappa(t)=t\) for \(t\ge1\).}
\(W_y:=\inf_{1\le b<\infty}b/F_y(b)\).
Since \(\pKtt^\kappa_{M,\delta}(\ACC\mid y)=\log_2q_{y,\delta}\),
we obtain the following connection:\footnote{We thank Alessandro
Achille for pointing out this identity and for explaining how the
universal restart results discussed in their paper apply to our
verified setting.}
\begin{equation}\label{eq:pkt-single-work}
2^{\,\mathrm{pk}^\kappa_M(\ACC\mid y)}
=
\inf_{0<\delta\le1}\frac{q_{y,\delta}}{\delta}
=
\inf_{1\le b<\infty}\frac{b}{F_y(b)}
=
W_y,
\end{equation}
with division by zero interpreted as \(+\infty\).
The first equality follows because exponentiation commutes with
taking an infimum and
\(\pKtt^\kappa_{M,\delta}(\ACC\mid y)=\log_2 q_{y,\delta}\).
As Achille and Soatto~\cite{AS26} explain, classical universal
restart strategies~\cite{LSZ93} allow one to obtain a verified
answer with expected running time within a logarithmic factor
of this search cost, without knowing the best computation budget.
This makes \(W_y\) a natural measure of search cost.

Unfortunately, as noted in the introduction, the difficulty with
\(W_y\), and more generally with defining \(\mathrm{pk}^\kappa_M\)
by optimizing over all quantile levels, is that its value can
depend on exceptionally rare favorable thinking routes.
Estimating it may therefore require exponentially many samples,
and hence exponential time. Fixing \(\delta\) avoids this
optimization over arbitrarily rare events and gives the
estimation guarantee of Theorem~\ref{thm:estimate}.

One can still combine several quantile levels while retaining an
estimation guarantee by minimizing over a fixed finite grid,
for example \(\Delta_J=\{2^{-1},\ldots,2^{-J}\}\) for \(J\ge1\):
\begin{equation}\label{eq:pkt-grid}
\mathrm{pk}^\kappa_{M,\Delta_J}(z\mid y)
:=
\min_{\delta\in\Delta_J}
\left\{
\pKtt^\kappa_{M,\delta}(z\mid y)+\log_2\frac1\delta
\right\}.
\end{equation}
For a fixed finite grid of constant quantile levels,
Theorem~\ref{thm:estimate} and a union bound give simultaneous
quantile-bracketing guarantees, and hence bounds on
\eqref{eq:pkt-grid}. However, even the exact grid-restricted measure
can be arbitrarily far from the unrestricted measure
\(\mathrm{pk}^\kappa_M\).

\section{An Experimental Illustration}\label{sec:experiment}

We illustrate the measure on twelve problems from GSM8K, a dataset of
grade-school mathematics word problems with step-by-step natural-language
reference solutions~\cite{GSM8K}.
The experiment illustrates, on
a small scale, why thinking and computation must be incorporated rather than
using only the probability-based, non-thinking prompt value of
Section~\ref{sec:exact}. It also illustrates why prompt value is naturally
indexed by a quantile: the same prompt may help one part of the rollout
distribution while hurting another. 
(The experiment is intended as an
illustration of the measure, not as a population-level evaluation of prompting
on GSM8K.)

\paragraph{Problems and prompts.}
We sample \(100\) problems without replacement from the
GSM8K training split and retain the first twelve whose reference solutions
contain at least three newline-delimited steps. For each problem \(q\), the prompt \(p\) is the first reference step, with
GSM8K calculator annotations such as \texttt{<<2*300=600>>} removed. For
example, the first selected problem, \texttt{gsm8k\#4205}, is:
\begin{quote}
\emph{While at the lake, Cohen saw 300 fish-eater birds that had migrated into
the area recently walking by the lake, eating the fish they had caught. The
number of birds at the lake doubled on the second day and reduced by 200 on the
third day. How many fish-eater birds did Cohen see in the three days?}
\end{quote}
The canonical answer is \(1300\), and the first reference step is:
\begin{quote}
\emph{Since there were 300 fish-eater birds in the lake on the first day, the
number doubled to \(2*300=600\) birds on the second day.}
\end{quote}

The accompanying online notebook\footnote{\url{https://www.kaggle.com/code/rafaelpass/the-value-of-a-prompt}}
provides the complete code needed to reproduce the problem selection, sampled
routes, scoring output, and figures.

\paragraph{Protocol.}
The underlying model is
\texttt{deepseek-ai/DeepSeek-R1-Distill-Qwen-1.5B}~\cite{DeepSeekR1}. We load
the model in FP16, a standard GPU-inference configuration, and convert its logits to FP32 before computing probabilities. Following DeepSeek's benchmark
configuration, we use temperature \(0.6\) and draw \(64\) independent rollouts
for every problem in each of two conditions. DeepSeek's configuration uses
top-\(p=0.95\), which discards low-probability tokens outside the sampling
nucleus and renormalizes the remaining distribution. 
We instead set top-\(p=1\), so that no such truncation occurs and every token
remains in the support of the sampling distribution.

Both conditions (prompted and unprompted) contain the same GSM8K question and answer instruction:
``Give the final answer as a single number on the last line, in the form
\texttt{ANSWER: <number>}.'' 
After the question, the model context ends
with \texttt{<think>}, the model's start-of-thinking marker. In the condition
\(y=\eps\), we sample the model's thinking immediately. In the condition
\(y=p\), we first append the reference step after \texttt{<think>} and then
sample the model's continuation. Thus the reference step is supplied as partial
computation.\footnote{In a preliminary comparison, supplying the same reference step in
the ``user turn" had much less effect on the model's success. In that placement,
the model may treat the step as an assertion to verify or re-derive. We instead
prefill it inside the open thinking block so that the model can continue from
it as already supplied partial computation, which is the intervention we wish
to value.}

We next sample a thinking route \(H\), stopping when the model emits
\(\EOT\) or \(\EOS\), or when an experimental horizon is reached.\footnote{In our experiment set-up, the experimental horizon is set to \(1800\) thinking tokens; all sampled rollouts terminate before reaching it.} For each
possible stopping time \(t\), we retain \(H_{\leq t}\), append \(\EOT\), and
supply the fixed field \texttt{ANSWER:}. The artifact \(z\) is the continuation
consisting of a leading space followed by the canonical gold numeral; we
evaluate the probability of emitting \(z\) and then \(\EOS\). (The fixed
\texttt{ANSWER:} field is excluded
from the thinking token count going into cost \(\kappa\).)

\paragraph{Cost conventions and the estimator.}
We evaluate both token-equivalent cost conventions introduced after
Definition~\ref{def:kappa}. Recall that
\(\kappa^{\mathrm{gen}}\) charges the supplied context at the prefill rate and
the thinking prefix and artifact at the sequential-decoding rate,
whereas \(\kappa^{\mathrm{pre}}\) charges the supplied context and thinking
prefix at the prefill rate and only the artifact at the
sequential-decoding rate. 
We normalize \(c_{\mathrm{pre}}=1\), use
\(c_{\mathrm{dec}}=32\) as the default, and examine sensitivity over a sampled grid with \(c_{\mathrm{dec}}\in[8,256]\).

For every rollout \(H^y\), we evaluate every truncation index
\(t=0,\ldots,|H^y|\) and first minimize the objective defining
\(\widetilde{Kt}^{\kappa}_M(z\mid y;H^y)\) over \(t\). 
We then let \(\widehat m_{64}^{\,y}\) denote the empirical lower
\(\delta\)-quantile of the \(64\) routewise minima for
\(y\in\{\eps,p\}\). Following the protocol of
Section~\ref{sec:protocol}, the reported prompt-value estimate is
$\widehat{\Val}^{\kappa}_{M,\delta,64}(p;z)
  =
  \widehat m_{64}^{\,\eps}-\widehat m_{64}^{\,p},$
the empirical counterpart of
\(\widetilde{\Val}^{\kappa}_{M,\delta}(p;z)\).

\subsection{Results and observations}\label{sec:experiment-results}

For presentation, we divide the twelve problems into two cohorts under the
generated-thought convention with \(c_{\mathrm{dec}}=32\). The first contains
the six problems whose estimated prompt value is positive at all three marked
quantiles, \(\delta\in\{0.2,0.5,0.8\}\); the second contains the remaining six.

\paragraph{Thinking can reverse the non-thinking verdict.}
Figure~\ref{fig:description-length-positive} isolates the probability component
of the measure for the cohort positive at all three marked quantiles. In five of the six cases,
supplying the correct first step initially makes the gold artifact \emph{less}
likely at \(t=0\) (and in the sixth one only slightly improves). Nevertheless, once thinking is
incorporated, the generated-thought prompt value is positive at all three
marked quantiles. The curves show the source of this reversal: the prompted
condition reaches favorable artifact probabilities after shorter thinking
prefixes. Thus the non-thinking probability-based value of
Section~\ref{sec:exact}, including the corresponding criterion of
\cite{Xie26}, can give the \emph{opposite} qualitative verdict.

\begin{figure*}[!tp]
  \centering
  \includegraphics[width=0.84\textwidth]
  {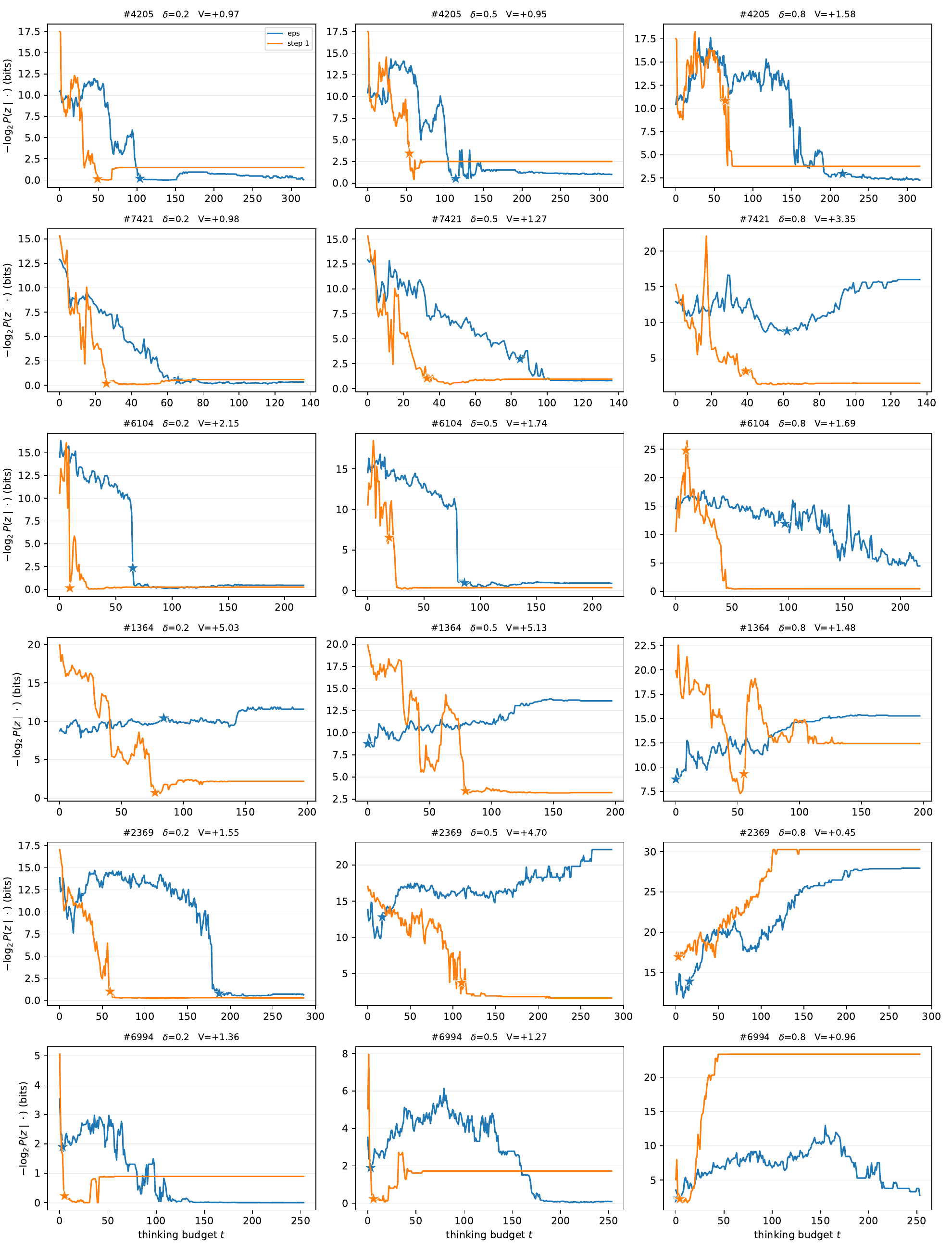}
  \caption{Probability-only artifact-description-length profiles for the six
  cases positive at all three marked quantiles. Columns correspond to
  \(\delta\in\{0.2,0.5,0.8\}\); blue denotes the unprompted condition and
  orange the prompted condition. Each curve is the pointwise empirical lower
  \(\delta\)-quantile of
  \(-\log_2 P_M(z\mid y,H^y_{\leq t},\EOT)\), without a computational charge.
  Stars mark the minimizing \(t\)'s in the generated-thought
  \(\widetilde{Kt}\) objective for the rollouts defining the displayed
  empirical quantiles.}
  \label{fig:description-length-positive}
\end{figure*}

\paragraph{Acceleration v.s. Steering:}
Figure~\ref{fig:pv-cost-comparison} evaluates the cohort positive at all three marked quantiles
under both cost conventions. For several cases, the generated-thought value is
positive while the prefix-prefill value is close to zero. This attenuation is
consistent with the prompt primarily \emph{accelerating} computation that the
unprompted model can recover by using a longer thinking prefix once that prefix
is charged only at the cheaper prefill rate.

In other cases, a substantial advantage persists under prefix-prefill
accounting. Together with the description-length curves in
Figure~\ref{fig:description-length-positive}, this persistence suggests that,
the gain is not explained solely by avoiding
sequential thinking. It is instead consistent with the prompt \emph{steering} the model toward a state that
assigns greater probability to the artifact. 

\begin{figure*}[!tp]
  \centering
  \includegraphics[width=0.94\textwidth]
  {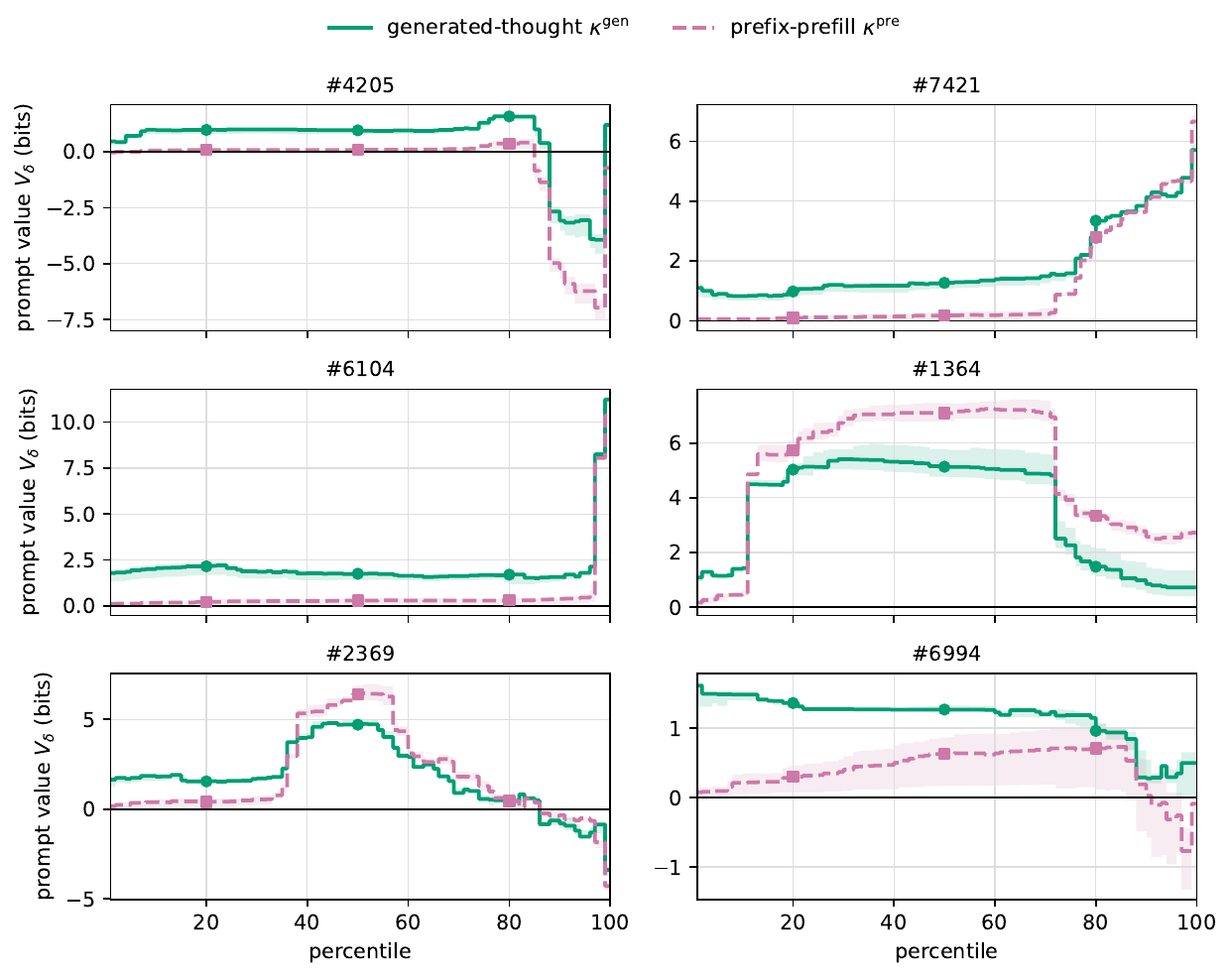}
  \caption{Prompt-value profiles for the cohort positive at all three marked quantiles under
  generated-thought accounting (solid green) and prefix-prefill accounting
  (dashed purple), with \(c_{\mathrm{dec}}=32\). Circles and squares mark
  \(\delta\in\{0.2,0.5,0.8\}\), and positive values favor the prompted
  condition. The green and purple shaded regions are the corresponding
  pointwise envelopes over the sampled
  \(c_{\mathrm{dec}}\in[8,256]\) grid; they are cost-sensitivity envelopes,
  not confidence bands. Vertical scales vary by panel.}
  \label{fig:pv-cost-comparison}
\end{figure*}

\begin{samepage}
\paragraph{A correct partial solution need not be valuable.}
Under the default generated-thought convention, the first reference step has
positive estimated value at all three marked quantiles in six of the twelve
cases.
Figure~\ref{fig:pv-generated-other} shows that, among the \emph{remaining} six,
the effects range from harmful to negligible or mixed. This is not surprising: a correct step from a human
reference solution was not designed as an optimal prompt for this model and
may be redundant or induce an unfavorable continuation. Correctness of the
supplied reasoning therefore does not by itself imply positive prompt value.

\paragraph{Prompt value is distribution-dependent.}
Several generated-thought profiles in
Figure~\ref{fig:pv-generated-other} cross zero as \(\delta\) varies, and
the crossings occur in both directions. The same prompt can therefore help one
part of the rollout distribution while hurting another, so its value at the
median need not describe its effect elsewhere in the distribution. Reporting
prompt value as a function of \(\delta\), rather than at a single quantile,
makes this heterogeneity visible.

\end{samepage}

\begin{figure*}[!tp]
  \centering
  \includegraphics[width=0.94\textwidth]
  {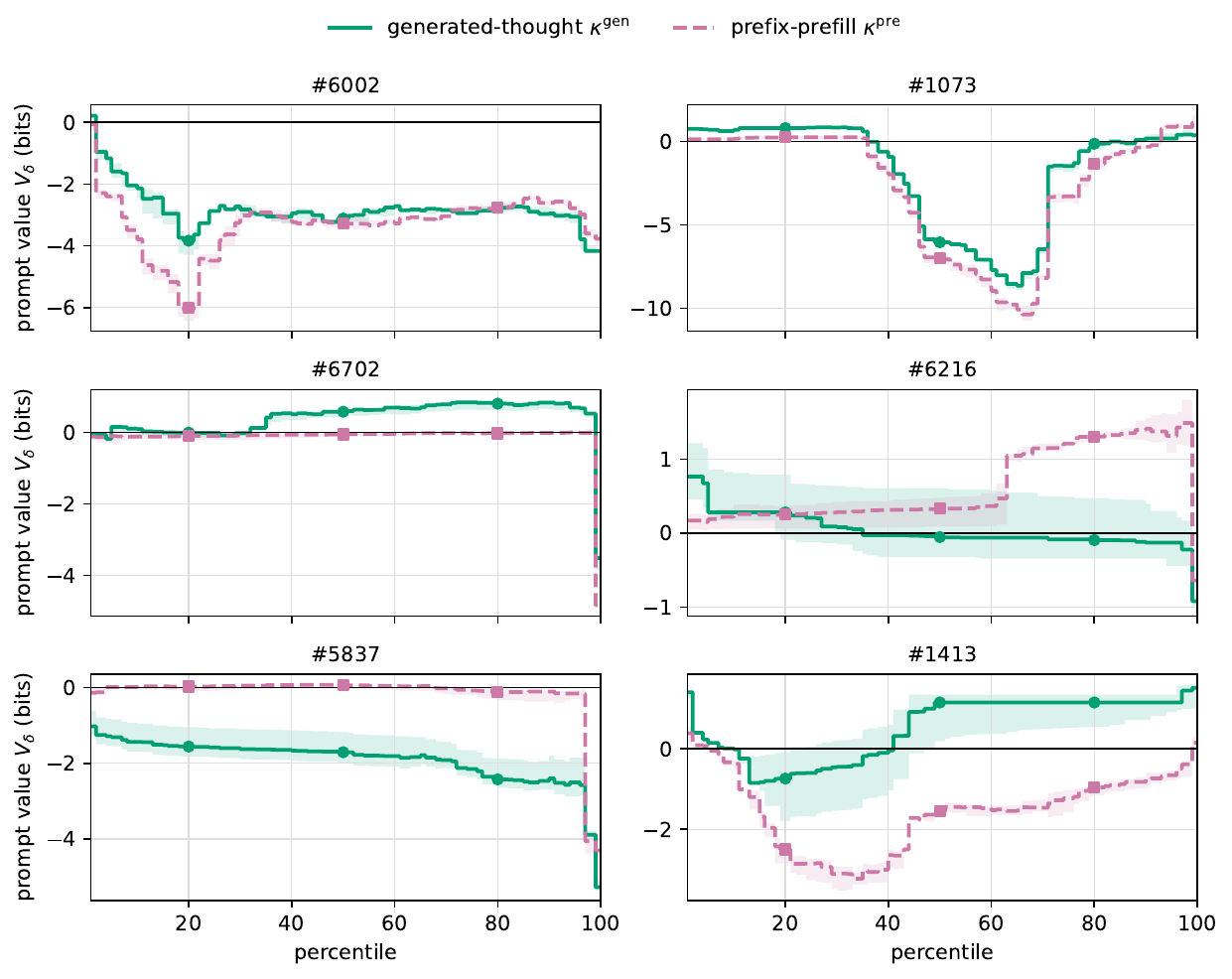}
  \caption{Prompt-value profiles for the remaining six problems under
  generated-thought accounting (solid green) and prefix-prefill accounting
  (dashed purple), with \(c_{\mathrm{dec}}=32\). Circles and squares mark
  \(\delta\in\{0.2,0.5,0.8\}\), and positive values favor the prompted
  condition. The green and purple shaded regions are the corresponding
  pointwise envelopes over the sampled
  \(c_{\mathrm{dec}}\in[8,256]\) grid; they are cost-sensitivity envelopes,
  not confidence bands. Vertical scales vary by panel.}
  \label{fig:pv-generated-other}
\end{figure*}

\clearpage

\section{Related work}\label{sec:related}

\paragraph{Algorithmic information, resource bounds, and compression.}
Our starting point is Kolmogorov's classical notion of algorithmic information
\cite{Kolmogorov65,ZL70}; see \cite{LiVitanyi} for a modern treatment. A
second classical ingredient is \emph{a-priori complexity}: For a universal
prefix machine \(U\), let \(m_U(x)\) denote the probability that \(U\) outputs
\(x\) when its input bits are sampled uniformly. Levin's coding theorem
identifies \(-\log_2m_U(x)\), up to an additive constant, with the prefix-free
Kolmogorov complexity of \(x\) \cite{ZL70,LiVitanyi}. Thus two classical views
of complexity---the length of a shortest prefix-free program and the negative
logarithm of universal generation probability---agree up to an additive
constant. 
Our program-based and a-priori LLM-relative notions instantiate these two
classical views for a fixed model, and the relationship between them established
in Theorem~\ref{thm:semantics} is the corresponding fixed-model analogue of
the coding theorem.

To incorporate computation, we draw on Levin's \(Kt\) complexity, which
combines program length with the logarithm of running time
\cite{Levin73}. 

The use of next-symbol probabilities from neural language models for lossless
compression predates modern LLMs \cite{SchmidhuberHeil96}; recent LLM-based
examples include \cite{Valmeekam23,Deletang24}.

\paragraph{Probabilistic and randomized Kolmogorov complexity.}
Our $\pKtt^{\kappa}_M$ is the model-relative analog of a line of work on
Kolmogorov complexity ``relative to a random tape". In particular, Goldberg, Kabanets, Lu and
Oliveira define the \emph{probabilistic} time-bounded Kolmogorov complexity
$pK^t_\delta(x)$ by
\[
pK^t_\delta(x)
=
\min\Bigl\{
k :
\Pr_r\!\bigl[K^t(x\mid r)\le k\bigr]\ge \delta
\Bigr\},
\]
where the probability is over the random tape $r$ \cite{GKLO22}. Thus,
$pK^t_\delta(x)$ is simply the $\delta$-quantile of the random variable
$K^t(x\mid r)$. \cite{HLO25} extend this notion to consider Levin's $Kt$-complexity instead of $K^t$, resulting in the notion of $pKt$.

One may regard the realized thought \(H^y\) as generated from an
underlying uniformly random tape \(r\). Under this representation, our
definition has exactly the same structure as \(pKt_\delta\); our
specializations are that the reference machine is the LLM rather than a universal Turing machine, and that we use the a-priori companion of $Kt$ complexity in place of $K^t$/$pKt$. 

This notion is to be contrasted with \emph{randomized} Kolmogorov complexity
$rK^t$ \cite{LOS21}, where the program is fixed \emph{before} the randomness
and must succeed with high probability over it; perhaps surprisingly, this notion is less amenable for our purposes.

\paragraph{PMI and prompt scoring.}
As mentioned, the likelihood-ratio expression $\log_2 P_M(z \mid p)/P_M(z)$
obtained in our non-thinking prompt value has the exact
algebraic form of pointwise mutual information (PMI); PMI originates in classical information theory with Fano~\cite{Fano61} and
was popularized in computational linguistics by Church and Hanks as a measure
of association between particular word pairs~\cite{CH90}.

As discussed, Xie et al.~\cite{Xie26} study the closely related
problem of measuring human contribution in AI-assisted content
generation. For a human input \(x\) and LLM output string \(y\),
they define ``self-information'' and ``conditional self-information''
by \(I(y)=-\log p_\theta(y)\) and
\(I(y\mid x)=-\log p_\theta(y\mid x)\), and use the normalized
``contribution score''
\(\phi=(I(y)-I(y\mid x))/I(y)\).
The unnormalized numerator
\(I(y)-I(y\mid x)=\log(p_\theta(y\mid x)/p_\theta(y))\)
is exactly PMI and thus coincides with the no-thinking case of our
framework. Achille and Soatto~\cite{AS26} use the same unnormalized
score for a specified computation sequence, which in their LLM
setting can be a sequence of thinking tokens. They call it
model-relative algorithmic mutual information and interpret it
as the logarithmic reduction in the search cost for that sequence.

Sorensen et al.\ use Shannon mutual information to select prompt templates,
maximizing mutual information between task inputs and model outputs over an
unlabeled evaluation distribution \cite{Sorensen22}. Their objective ranks
templates by averaging across task instances, whereas ours measures the value
of a particular prompt for producing a particular artifact.

Cosma~\cite{Cosma26} considers an alternative LLM-relative notion of Kolmogorov complexity, referred to as
\emph{prompting complexity}: the length of the shortest prompt that
makes a fixed LLM produce a target text under \emph{deterministic} decoding.
In other words, Cosma takes the prompt to be the program and views a deterministic LLM as the underlying reference machine. We take the opposite view: the LLM together with the given prompt is the reference
machine, while a specification of the model's sampling randomness is
the program. A consequence of this distinction is that Cosma's notion is only known to be exponential-time computable (through enumeration of all prompts), whereas our
LLM-relative notion of Kolmogorov-complexity is computable in polynomial time.

\paragraph{Search costs for verifiable tasks.}
Achille and Soatto~\cite{AS26} consider the search cost
\(\inf_{t\ge1}t/F(t)\), where \(F(t)\) is the probability of finding
a verified answer within \(t\) computation steps. This is a classical
restart benchmark~\cite{LSZ93}; time/success ratios are also standard
in cryptography~\cite{Lub96,MP18}.
As Achille and Soatto explain, classical universal restart strategies
allow one to obtain a verified answer with expected computation
within a logarithmic factor of this cost, without knowing the best
budget \(t\).
One could define prompt value by comparing the logarithms of these
search costs with and without the prompt. Estimating such a value, however,
may require exponentially many samples because it can depend on
exceptionally rare successful thinking routes.
(As discussed in Section \ref{sec:tokencost}, in the verified setting,
adding a rarity charge to our notion and  optimizing over
all quantile levels recovers the logarithm of this search cost.)

Earlier work on search cost includes Solomonoff's study of sequential search~\cite{Sol84};
Schmidhuber's Speed Prior~\cite{Sch02} provides another way to incorporate
computation time into probabilities.

\paragraph{Relationship with watermarking schemes.}
The view that sampling
randomness can be coupled to generated text has precedents in language-model
watermarking, where distribution-preserving and cryptographic schemes map
random keys or streams through autoregressive samplers
\cite{Kuditipudi23,CGZ24}. Our dyadic-program construction is not a
watermarking method, but it uses the same basic fact that an autoregressive
model together with a random real determines a generated string.

\paragraph{Value of information, costly computation, and AI economics.}
Classical value-of-information theory measures the improvement in expected
utility obtained from observing a signal~\cite{Howard66}, while work on bounded
rationality treats computation itself as costly~\cite{RW91}. Our approach
follows most directly the \emph{value of computational information} perspective
of Halpern and Pass~\cite{HP11,HP14}. In their framework, an explicit utility
function evaluates both the action (i.e., the outcome) produced and the computation
used to produce it; information may therefore be valuable not only because it
improves the outcome, but also because it saves computation. Our measure
specializes this perspective to a fixed LLM $M$ and fixed artifact $z$: rather than
assigning an external utility to \(z\), it measures how much the prompt reduces
the model-relative cost of producing \(z\).

Our comparison between what can be produced with and without the prompt also follows
the simulation paradigm underlying zero-knowledge proofs~\cite{GMR}. Halpern
and Pass relate this perspective to their framework for valuing computational
information and conversation, notably through the notion of \emph{precise zero
knowledge} introduced by Micali and Pass~\cite{MP06,HP11}. Their framework gives
a general, utility-dependent characterization of when information is valuable,
but its evaluation involves optimizing expected utility over possible
algorithms and it does not provide a general efficient estimator. Our
fixed-model notion is narrower, but this narrowing is what enables efficiently estimability.

More broadly, our work is motivated by the question of where human value lies
as AI makes ``creation" cheap. Catalini, Hui, and Wu provide a
complementary answer, emphasizing human verification---the scarce capacity to
check, audit, and assume responsibility for AI-generated outputs~\cite{CHW26}.
We focus on another form of human contribution: providing an input that helps
the model produce an artifact, and ask how the value of that input can be
measured.

\section{Conclusions and Future Work}\label{sec:conclusion}

We defined an LLM-relative prompt value that credits both changes in artifact
likelihood and reductions in required thinking, showed how to estimate its
quantiles from sampled routes, and gave it an operational reproduction-cost
interpretation. The GSM8K illustration shows that a non-thinking likelihood
comparison can reverse once thinking is admitted, and that the resulting value
depends both on the rollout quantile and on how thinking is costed.

\paragraph{The value of conversation.}
Our framework values a single prompt supplied in a fixed context. For a
realized multi-turn conversation, one can apply the measure sequentially,
valuing each human input conditional on the transcript at which it arrives and
summing the resulting increments. This provides an ex post accounting of the
realized contributions. It does not, however, account for the computation used
to formulate those inputs or for the adaptivity of the human's strategy, and
therefore need not capture the value of access to the underlying conversational
policy.

Halpern and Pass address this more general question by specifying an interactive
Turing machine that conducts the conversation~\cite{HP11}. That approach is
less suitable in our setting, where the evaluator generally observes the
human's messages but has no description of the adaptive policy that produced
them. Defining and measuring the value of adaptive human contributions under
such limited access is an intriguing open question.

\paragraph{The value of an artifact.}
Our framework measures the value of a prompt for producing a declared artifact.
A related question is whether it can help evaluate the artifact itself. The
unprompted quantity
\[
\pKtt^\kappa_{M,\delta}(z\mid\eps)
\]
measures how difficult \(z\) is for the reference LLM to produce, but this
cannot by itself be interpreted as the artifact's value: a random string may
have high production difficulty while having no substantive value.

One possible direction is \emph{semantic re-randomization}: apply a declared,
prompt-independent procedure---for example, instruct another LLM to rewrite
\(z\) while preserving its meaning---and compute the unprompted complexity of
the resulting rewrite. One might aggregate this quantity over several
independent rewrites, with the hope that semantically inert details disappear
while the difficulty of the underlying content remains. We leave an exploration of this for future work.

\section{Acknowledgments}
I am very grateful to Noam Mazor for helpful discussions. I am also very grateful to Alessandro Achille and Stefano Soatto for bringing their work
and related search literature to my attention, and Alessandro Achille for pointing out the connections between
our quantile formulation in the verified setting, their notion
of proper time, and universal restart strategies.

As mentioned above,
I am also very grateful to ChatGPT and Claude for extensive discussions,
drafting, editing, reviewing and, notably, implementing all the experiments.
\bibliographystyle{alpha}
\bibliography{references-3}

\appendix

\section{Proof of Theorem~\ref{thm:semantics}}
\label{app:dyadic}

We start with two simple facts about dyadic subintervals.

\begin{lemma}[Dyadic subinterval lemma]\label{lem:subinterval}
Every half-open interval $J=[a,b)\subseteq[0,1)$ of positive length
$\ell=b-a$ contains a dyadic interval $D_\pi$ of length greater than
$\ell/4$.
\end{lemma}

\begin{proof}
Let $j\ge 0$ be the largest integer such that $
2^{-j}>\frac{\ell}{2}$
and set $s:=2^{-(j+1)}$. By maximality of $j$,
$s\le \frac{\ell}{2}$
while $2s=2^{-j}>\ell/2$, so $
s>\frac{\ell}{4}$.

Now let $k:=\lceil a/s\rceil$. Then
\[
ks\ge a
\qquad\text{and}\qquad
ks<a+s.
\]
Hence
\[
(k+1)s<a+2s\le a+\ell=b.
\]
Therefore
\[
[ks,(k+1)s)\subseteq[a,b).
\]
Moreover, $(k+1)s<b\le1$, so $0\le k<2^{j+1}$; hence $k$ has a
$(j+1)$-bit representation $\pi$ and $[ks,(k+1)s)=D_\pi$. This is a dyadic
interval of length $s>\ell/4$.
\end{proof}

\begin{lemma}[Dyadic interval bound]\label{thm:dyadic}
Let $J = [a,b) \subseteq [0,1)$ be a half-open interval of positive length
$|J| = b-a$. Define $K(J) := \min\{ |\pi| : D_\pi \subseteq J \}$. Then
\begin{equation}\label{eq:dyadicbound}
-\log_2 |J| \;\le\; K(J) \;<\; -\log_2 |J| + 2.
\end{equation}
\end{lemma}

\begin{proof}
For the lower bound, suppose $D_\pi \subseteq J$. Then
$2^{-|\pi|} = |D_\pi| \le |J|$; taking $-\log_2$ gives $|\pi| \ge -\log_2 |J|$.
Since this holds for every valid $\pi$, it holds for $K(J)$.

For the upper bound, Lemma~\ref{lem:subinterval} gives a dyadic interval
$D_\pi \subseteq J$ with $|D_\pi| > |J|/4$, hence
$2^{-|\pi|} > |J| \cdot 2^{-2}$; taking $-\log_2$ gives
$|\pi| < -\log_2 |J| + 2$, and therefore $K(J) \le |\pi| < -\log_2 |J| + 2$.
\end{proof}

\paragraph{Returning to the Proof of Theorem~\ref{thm:semantics}}
\begin{proof}[Proof of Theorem~\ref{thm:semantics}]
For part (i), fix a context $y$ and write $q := P_M(z \mid y) > 0$. By
\eqref{eq:interval}, the interval $J := I_y(z)$ has length $q$, and by
Definition~\ref{def:KM}, $K_M(z \mid y) = K(J)$. Lemma~\ref{thm:dyadic} gives
\[
-\log_2 q \;\le\; K_M(z \mid y) \;<\; -\log_2 q + 2,
\]
and adding $\log_2 q$ throughout yields
$0 \le K_M(z \mid y) + \log_2 q < 2$, which is \eqref{eq:codinggap}. 

For part (ii), by part (i),  there is, for each context $y$, a number
$\gamma_y(z) \in [0,2)$ such that
$K_M(z \mid y) = -\log_2 P_M(z \mid y) + \gamma_y(z)$. Therefore
\[
\Val_M(p;z) = \widetilde{\Val}_M(p;z) + \gamma_\eps(z) - \gamma_p(z),
\]
and the difference of two numbers in $[0,2)$ lies in $(-2,2)$, proving
\eqref{eq:threebits}.
\end{proof}

\end{document}